\documentclass{article} 
\usepackage{iclr2021_conference,times}
\usepackage{amsmath}
\usepackage{amsthm}
\usepackage{makecell}
\usepackage{wrapfig} 
\usepackage{url}
\usepackage{algorithm,algorithmic}
\usepackage{amssymb, multirow, paralist, color}
\usepackage{graphicx,graphics}
\newtheorem{thm}{Theorem}

\newtheorem{lemma}{Lemma}

\newtheorem{ass}{Assumption}
\usepackage{enumitem}
\usepackage{textcase}
\usepackage[T1]{fontenc}
\usepackage[colorlinks=true,linkcolor=red,citecolor=blue]{hyperref}

\def \R {\mathbb{R}}

\def \w {\mathbf{w}}
\def \v {\mathbf{v}}

\def \x {\mathbf{x}}

\def \E {\mathrm{E}}

\def \x {\mathbf{x}}

\def \1 {\mathbf{1}}

\def \z {\mathbf{z}}

\def \y {\mathbf{y}}

\def \g {\mathbf{g}}

\def \wh {\widehat{\w}}

\def \y {\mathbf{y}}
\def \E {\mathrm{E}}
\def \x {\mathbf{x}}

\def \z {\mathbf{z}}

\def \w {\mathbf{w}}

\def \R {\mathbb{R}}

\def \v {\mathbf{v}}

\def \wh {\widehat{\w}}

\def \E {\mathbb{E}}

\title{Adam$^+$: A Stochastic Method with Adaptive Variance Reduction}

\author{Mingrui Liu\\
Boston University\\
\texttt{mingruiliu.ml@gmail.com}
\And
Wei Zhang\\
IBM T. J. Watson Research Center\\
\texttt{weiz@us.ibm.com}
\And
Francesco Orabona\\
Boston University\\
\texttt{francesco@orabona.com}
\And
Tianbao Yang\\
University of Iowa\\
\texttt{tianbao-yang@uiowa.edu}
}
\iclrfinalcopy 
\begin{document}

\maketitle
\vspace*{-0.3in}
\begin{abstract}
\vspace*{-0.1in}
Adam is a widely used stochastic optimization method for deep learning applications. While practitioners prefer Adam because it requires less parameter tuning, its use is problematic from a theoretical point of view since it may not converge.  Variants of Adam have been proposed with provable convergence guarantee, but they tend not be competitive with Adam on the practical performance. In this paper, we propose a new method named Adam$^+$ (pronounced as Adam-plus). Adam$^+$ retains some of the key components of Adam but it also has several noticeable differences: (i) it does not maintain the moving average of second moment estimate but instead computes the moving average of first moment estimate at extrapolated points; (ii) its adaptive step size is formed not by dividing the square root of second moment estimate but instead by dividing the root of the norm of first moment estimate. As a result,  Adam$^+$ requires few parameter tuning, as Adam, but it enjoys a provable convergence guarantee. Our analysis further shows that Adam$^+$ enjoys adaptive variance reduction, i.e., the variance of the stochastic gradient estimator reduces as the algorithm converges, hence enjoying an adaptive convergence. We also propose a more general variant of Adam$^+$ with different adaptive step sizes and establish their fast convergence rate.  
Our empirical studies on various deep learning tasks, including image classification, language modeling, and automatic speech recognition, demonstrate that Adam$^+$ significantly outperforms Adam and achieves comparable performance with best-tuned SGD and momentum SGD.

\end{abstract}
\vspace*{-0.2in}
\section{Introduction}
\vspace*{-0.1in}
Adaptive gradient methods~\citep{duchi2011adaptive,mcmahan2010adaptive,tieleman2012lecture,kingma2014adam,reddi2019convergence} are one of the most important variants of Stochastic Gradient Descent (SGD) in modern machine learning applications. Contrary to SGD, adaptive gradient methods typically require little parameter tuning still retaining the computational efficiency of SGD. One of the most used adaptive methods is Adam~\citep{kingma2014adam}, which is considered by practitioners as the de-facto default optimizer for deep learning frameworks. Adam computes the update for every dimension of the model parameter through a moment estimation, i.e., the estimates of the first and second moments of the gradients. The estimates for first and second moments are updated using exponential moving averages with two different control parameters. These moving averages are the key difference between Adam and previous adaptive gradient methods, such as Adagrad~\citep{duchi2011adaptive}. 

Although Adam exhibits great empirical performance, there still remain many mysteries about its convergence. First, it has been shown that Adam may not converge for some objective functions~\citep{reddi2019convergence,chen2018convergence}. Second,  it is unclear what is the benefit that the moving average brings from theoretical point of view, especially its effect on the convergence rate. Third, it has been empirically observed that adaptive gradient methods can have worse generalization performance than its non-adaptive counterpart (e.g., SGD) on various deep learning tasks due to the coordinate-wise learning rates~\citep{wilson2017marginal}. 

The above issues motivate us to design a new algorithm which achieves the best of both worlds, i.e., provable convergence with benefits from the moving average and enjoying good generalization performance in deep learning. Specifically, we focus on the following optimization problem:
\begin{align*}
	\min_{\w\in\R^d} \ F(\w),
\end{align*}
where we only have access to stochastic gradients of $F$.
Note that $F$ could possibly be nonconvex in $\w$. Due to the non-convexity, our goal is to design a stochastic first-order algorithm to find the $\epsilon$-stationary point, i.e., finding $\w$ such that $\E\left[\|\nabla F(\w)\|\right]\leq \epsilon$, with low iteration complexity.

Our key contribution is the design and analysis of a new stochastic method named Adam$^+$. Adam$^+$ retains some of the key components of Adam but it also has several noticeable differences: (i) it does not maintain the moving average of second moment estimate but instead computes the moving average of first moment estimate at extrapolated points; (ii) its adaptive step size is formed not by dividing the square root of coordinate-wise second moment estimate but instead by dividing the root of the norm of first moment estimate.  These features allow us to establish the adaptive convergence of Adam$^+$.  Different from existing adaptive methods where the adaptive convergence depends on the growth rate of stochastic gradients~\citep{duchi2011adaptive,mcmahan2010adaptive,kingma2014adam,luo2019adaptive,reddi2019convergence,chen2018closing,chen2018convergence,chen2018universal}, our adaptive convergence is due to the adaptive variance reduction property of our first order moment estimate.  
In existing literature, the variance reduction is usually achieved by large mini-batch~\citep{goyal2017accurate} or recursive variance reduction~\citep{fang2018spider,zhou2018stochastic,pham2020proxsarah,cutkosky2019momentum}. In contrast, we do not necessarily require large minibatch or computing stochastic gradients at two points per-iteration to achieve the variance reduction. In addition, we also establish a fast rate that matches the state-of-the-art complexity under the same conditions of a variant of Adam$^+$.   
Table~\ref{table:11} provides an overview of our results and a summary of existing results. There is another line of work focusing on designing algorithms without the knowledge of some hyperparameters but with same theoretical convergence guarantees as if these hyparameters were known in advance~\citep{li2019convergence,ward2019adagrad}. These schemes are usually referred to as algorithms with "adaptive stepsize". For example, the algorithms in~\citep{li2019convergence,ward2019adagrad} do not require the knowledge of the noise level in the stochastic gradient, and the Adagrad-Norm algorithm in~\citep{ward2019adagrad} does not need to know the smoothness parameter of the objective function. We refer readers to Section~\ref{sec:relatedwork} for a comprehensive survey of other related work.
We further corroborate our theoretical results with an extensive empirical study on various deep learning tasks.

\begin{table}[t]
\caption{Summary of different algorithms with different assumptions and complexity results for finding an $\epsilon$-stationary point. ``Individual Smooth'' means assuming that $F(\w)=\E_{\xi\sim\mathcal{D}}[f(\w; \xi)]$ and that every component function $f(\w; \xi)$ is $L$-smooth. ``Hessian Lipschitz'' means that $\|\nabla^2F(\x)-\nabla^2 F(\y)\|\leq L_H\|\x-\y\|$ holds for $\x,\y$ and $L_H\geq 0$. ``Type I'' means that the complexity depends on $\E[\sum_{i=1}^{T}\|\g_{1:T,i}\|]$, where $\g_{1:T, i}$ stands for the $i$-th row of the matrix $[\g_1, \dots, \g_T]$ with $\g_t$ being the stochastic gradient at $t$-th iteration and $T$ being the number of iterations. ``Type II'' means that complexity depends on $\E[\sum_{t=1}^{T}\|\z_t\|]$, where $\z_t$ is the variance-reduced gradient estimator at $t$-th iteration.}
\label{table:11}
 	\resizebox{\textwidth}{!}{	
	\begin{tabular}{|c|c|c|c|c|}
		\hline
		Algorithm                  & Individual Smooth & Hessian Lipschitz & Worst-case Complexity better than $O(\epsilon^{-4})?$                 & Data-dependent Complexity \\\hline
		\makecell{Generalized Adam~\citep{chen2018convergence}\\PAdam~\citep{chen2018closing}\\ Stagewise Adagrad~\citep{chen2018universal}\\
	} & No & No & No & Type I\\\hline
		\makecell{SPIDER~\citep{fang2018spider}\\ STORM~\citep{cutkosky2019momentum}\\  SNVRG~\citep{zhou2018stochastic}\\	Prox-SARAH~\citep{pham2020proxsarah}}
	 & Yes                   & No                & Yes                   & N/A                       \\\hline
			\makecell{SGD~\citep{fang2019sharp} \\ Normalized momentum SGD\\\citep{cutkosky2020momentum}}                       & No                    & Yes               & Yes & N/A                       \\\hline
	\makecell{Adam$^+$ (this work)}  & No & Yes & Yes & Type II \\\hline
	\end{tabular}}
\end{table}

Our contributions are summarized below.
\begin{itemize}
	\item We propose a new algorithm with adaptive step size, namely Adam$^+$, for general nonconvex optimization. We show that it enjoys a new type of data-dependent adaptive convergence that depends on the variance reduction property of first moment estimate. Notably, this data-dependent complexity does not require the presence of sparsity in stochastic gradients to guarantee fast convergence as in previous works~\citep{duchi2011adaptive,kingma2014adam,reddi2019convergence,chen2018universal,chen2018closing}. To the best of our knowledge, this is the first work establishing such new type of data-dependent complexity.
	\item We show that a general variant of our algorithm can achieve $O(\epsilon^{-3.5})$ worst-case complexity, which matches the state-of-the-art complexity guarantee under the Hessian Lipschitz assumption~\citep{cutkosky2020momentum}.
	\item We demonstrate the effectiveness of our algorithms on image classification, language modeling, and automatic speech recognition. Our empirical results show that our proposed algorithm consistently outperforms Adam on all tasks, and it achieves comparable performance with the best-tuned SGD and momentum SGD.
\end{itemize}
\vspace*{-0.15in}
\section{Algorithm and Theoretical Analysis}
\vspace*{-0.15in}
\begin{algorithm}[t]
	\caption{Adam$^+$: Good default settings for the tested machine learning problems are $\alpha = 0.1, a= 1$,
$\beta=0.1, \epsilon_0 = 10^{-8}$.}
	\label{Alg:1}
	\begin{algorithmic}[1]
		\STATE {\bf Require:}  $\alpha, a\geq 1$: stepsize parameters
		\STATE {\bf Require:}  $\beta\in(0,1)$: Exponential decay rates for the moment estimate
	   \STATE {\bf Require:}  $g_t(\w)$: unbiased stochastic gradient with parameters $\w$ at iteration $t$
	    \STATE {\bf Require:}  $\w_0$: Initial parameter vector
		\STATE $\z_0  = g_0(\w_0)$
		\FOR{$t=0,\ldots,T$}
		\STATE Set $\eta_t = \frac{\alpha \beta^{a}}{\max(\|\z_t\|^{1/2}, \epsilon_0)}$
		\STATE $\w_{t+1}   =\w_t - \eta_t\z_t$\label{line:1}
		\STATE $\wh_{t+1}  = (1-1/\beta)\w_t  + 1/\beta \cdot \w_{t+1}$\label{line:2}
		\STATE $\z_{t+1}  = (1-\beta)\z_t + \beta g_{t+1}(\wh_{t+1}) $\label{line:3}
		\ENDFOR
	\end{algorithmic}
\end{algorithm}
In this section, we introduce our algorithm Adam$^{+}$ (presented in Algorithm~\ref{Alg:1}) and establish its convergence guarantees. Adam$^+$ resembles Adam in several aspects but also has noticeable differences. Similar to Adam, Adam$^+$ also maintains an exponential moving average of first moment (i.e., stochastic gradient), which is denoted by $\z_t$, and uses it for updating the solution in line 8. However, the difference is that the stochastic gradient is evaluated on an extrapolated data point $\wh_{t+1}$, which is an extrapolation of two previous updates $\w_t$ and $\w_{t+1}$.  Similar to Adam,  Adam$^+$ also uses an adaptive step size that is proportional to $1/\|\z_t\|^{1/2}$. Nonetheless, the difference lies at its adaptive step size is directly computed from the square root of the norm of first moment estimate $\z_t$. In contrast, Adam uses an adaptive step size that is proportional to $1/\sqrt{\mathbf v_t}$, where $\v_t$ is an exponential moving average of second moment estimate. These two key components of Adam$^+$, i.e., extrapolation and adaptive step size from the root norm of the first moment estimate, make it enjoy two noticeable benefits: variance reduction of first moment estimate and adaptive convergence. We shall explain these two benefits later. 

Before moving to the theoretical analysis, we would like to make some remarks. First, it is worth mentioning that the moving average estimate with extrapolation is inspired by the literature of stochastic compositional optimization~\citep{wang2017stochastic}.
\citet{wang2017stochastic} showed that the extrapolation helps balancing the noise in the gradients, reducing the bias in the estimates and giving a faster convergence rate.
Here, our focus and analysis techniques are quite different. In fact, \cite{wang2017stochastic} focuses on the compositional optimization while we consider a general nonconvex optimization setting. Moreover, the analysis in~\citep{wang2017stochastic} mainly deals with the error of the gradient estimator caused by the compositional nature of the problem, while our analysis focuses on carefully designing adaptive normalization to obtain an adaptive and fast convergence rates. A similar extrapolation scheme has been also employed in the algorithm NIGT by \citet{cutkosky2020momentum}.
In later sections, we will also provide a more general variant of Adam$^+$ which subsumes NIGT as a special case.

Another important remark is that the update of Adam$^+$ is very different from the famous Nesterov's momentum method. In Nesterov's momentum method, the update of $\w_{t+1}$ uses the stochastic gradient at an extrapolated point $\wh_{t+1} = \w_{t+1} + \gamma(\w_{t+1} - \w_t)$ with a momentum parameter $\gamma\in (0,1)$. In contrast, in Adam$^+$ the update of $\w_{t+1}$ is using the moving average estimate at an extrapolated point $\wh_{t+1}= \w_{t+1}+ (1/\beta -1)(\w_{t+1} - \w_t)$. Finally, Adam$^{+}$ does not employ coordinate-wise learning rates as in Adam, and hence it is expected to have better generalization performance according to~\citet{wilson2017marginal}. 

\subsection{Adaptive Variance Reduction and Adaptive Convergence}
In this subsection, we analyze Adam$^+$ by showing its variance reduction property and adaptive convergence. To this end, we make the following assumptions. 
\begin{ass}
	\label{ass:1} There exists positive constants  $L, \Delta, L_H, \sigma$ and an initial solution $\w_0$ such that 
	\begin{itemize}
	\item[(i)] $F$ is $L$-smooth, i.e., $\left\| \nabla F(\x)-\nabla F(\y)\right\|\leq L\left\|\x-\y\right\|, \ \forall\x,\y\in\R^d$.
	\item [(ii)] For $\forall \x\in\R^d$, we have access to a first-order stochastic oracle at time $t$ $g_t(\x)$ such that $\E\left[g_t(\x)\right]=\nabla F(\x)$, $\E\left\|g_t(\x)-\nabla F(\x)\right\|^2 \leq \sigma^2$.
	\item [(iii)] $\nabla F$ is a $L_H$-smooth mapping, i.e., $\|\nabla ^2 F(\x)-\nabla^2 F(\y)\|\leq L_H\|\x-\y\|, \forall\x,\y\in\R^d$. 
		\item [(iv)]  $F(\w_0) - F_*\leq \Delta<\infty$, where $F_*=\inf_{\w\in\R^d} F(\w)$.

\end{itemize}
\end{ass}
\textbf{Remark}: Assumption~\ref{ass:1} (i) and (ii), (iv) are standard assumptions made in literature of stochastic non-convex optimization~\citep{ghadimi2013stochastic}.   Assumption (iii) is the assumption that deviates from typical analysis of stochastic methods. We leverage this assumption to explore the benefit of moving average, extrapolation and adaptive normalization. It is also used in some previous works for establishing fast rate of stochastic first-order methods for nonconvex optimization~\citep{fang2019sharp,cutkosky2020momentum} and this assumption is essential to get fast rate due to the hardness result in \citep{arjevani2019lower}. It is also the key assumption for finding a local minimum in previous works~\citep{carmon2018accelerated,agarwal2017finding,jin2017escape}.  
 
We might also assume that the stochastic gradient estimator in Algorithm~\ref{Alg:1} satisfies the following variance property.
\begin{ass}
	\label{ass:2}
	Assume that $\E[\|g_0(\w_0) - \nabla F(\w_0)\|^2]\leq \sigma_0^2$
	and $\E[\|g_t(\w_t) - \nabla F(\w_t)\|^2]\leq \sigma_m^2, t\geq 1$. 
\end{ass}
{\bf Remark: } When $g_0$ (resp. $g_t$) is implemented by a mini-batch stochastic gradient with mini-batch size $S$, then $\sigma_0^2$ (resp. $\sigma_m^2$) can be set as $\sigma^2/S$ by Assumption~\ref{ass:1} (ii).  We differentiate the initial variance and intermediate variance because they contribute differently to the convergence. 

We first introduce a lemma to characterize the variance of the moving average gradient estimator $\z_t$. 
\begin{lemma}
	\label{lemma:1}
	Suppose Assumption~\ref{ass:1} and Assumption \ref{ass:2} hold and $a\geq 1$. Then,
	there exists a sequence of random variables $\delta_t$ satisfying $\|\z_t - \nabla F(\w_t)\|\leq \delta_t$ for $\forall t\geq 0$, 
	\begin{align*}
		\E\left[\delta_{t+1}^2\right] &\leq \left(1- \frac{\beta}{2}\right)\E\left[\delta_t^2\right] +2\beta^2 \sigma_m^2 + \E\left[\frac{CL^2_H\|\w_{t+1} - \w_t\|^4}{\beta^3}\right]\\
		&\leq \left(1- \frac{\beta}{2}\right)\E\left[\delta_t^2\right] +2\beta^2 \sigma_m^2 + \E\left[CL_H^2\alpha^4\beta^{4a-3}\|\z_t\|^2\right], 
	\end{align*}
where $C=1944$.
\end{lemma}
\textbf{Remark}:  Note that $\delta_t$ is an upper bound of $\|\z_t-\nabla F(\w_t)\|$, the above lemma can be used to illustrate  the variance reduction effect for the gradient estimator $\z_t$. To this end, we can bound $\|\z_t\|^2\leq 2\delta_t^2 + 2\|\nabla F(\w_t)\|^2$, then the term $CL^2\alpha^4\beta^{4a-3}\delta_t^2$ can be canceled with $-\beta/4\delta_t^2$ with small enough $\alpha$. Hence, we have $\E\delta_{t+1}^2\leq (1-\beta/4)\E[\delta_t^2]+2\beta^2\sigma_m^2 + c\E[\|\nabla F(\w_t)\|^2]$ with a small constant $c$. As the algorithm converges with $\E[\|\nabla F(\w_t)\|^2]$ and $\beta$ decreases to zero, the variance of $\z_t$ will also decrease. Indeed, the above recursion of $\z_t$'s variance resembles that of the recursive variance reduced gradient estimators (e.g., SPIDER~\citep{fang2018spider}, STORM~\citep{cutkosky2019momentum}).
The benefit of using Adam$^+$ is that we do not need to compute stochastic gradient twice at each iteration. 



We can now state our convergence rates for Algorithm~\ref{Alg:1}.
\begin{thm}
	\label{thm1}
	Suppose Assumption~\ref{ass:1}  and Assumption \ref{ass:2} hold.  Suppose $\|\nabla F(\w)\|\leq G$ for any $\w\in\R^d$. By choosing the parameters such that  $\alpha^4\leq \frac{1}{36CL_H^2}$, $\alpha\leq \frac{1}{4L}$, $a= 1$ and $\epsilon_0=\beta^a$, we have
\begin{equation}
\label{thm1:eq3}
\begin{aligned}
\frac{1}{T}\sum_{t=1}^{T}\E\left\|\nabla F(\w_t)\right\|^2\leq \frac{G\E\left[\sum_{t=1}^{T}\|\z_t\|\right]}{T}+\frac{\Delta}{\alpha T}+\frac{18\sigma_0^2}{\beta T} + 30\beta \sigma_m^2~. 
\end{aligned}
\end{equation}
In addition, suppose the initial batch size is $T_0$ and the intermediate batch size is $m$, and choose $\beta=T^{-b}$ with $0\leq b\leq 1$, we have
\begin{equation}
\label{thm:eq4}
    \frac{1}{T}\sum_{t=1}^{T}\E\left\|\nabla F(\w_t)\right\|^2\leq \frac{\E\left[G\sum_{t=1}^{T}\|\z_t\|\right]}{T}+\frac{\Delta}{\alpha T}+\frac{18\sigma^2}{T^{1-b} T_0} + \frac{30 \sigma^2}{mT^b}~.
\end{equation}
\end{thm}

\begin{thm}
\label{thm:2}
	Suppose Assumption~\ref{ass:1} and Assumption \ref{ass:2} hold. 
	By choosing parameters such that  $640\alpha^3 L_H^{3/2}\leq 1/120$, $a=1, \epsilon_0=0, \beta=1/T^{s}$ with $s=2/3$
	then it takes 
$T=O\left(\epsilon^{-4.5}\right)$
	number of iterations to ensure that 
	\begin{equation*}
\frac{1}{T}\sum_{t=1}^{T} \E\left[\|\nabla F(\w_t)\|^{3/2}\right]\leq\epsilon^{3/2},\quad \frac{1}{T}\E\bigg[\sum_{t=1}^T\delta_t^{3/2}\bigg]\leq \epsilon^{3/2}~.
	\end{equation*}
\end{thm}
\vspace*{-0.1in}
\textbf{Remarks:} 
\vspace*{-0.1in}
\begin{itemize}
\item From Theorem~\ref{thm1}, we can observe that the convergence rate of Adam$^{+}$ crucially depends on the growth rate of $\E\left[\sum_{t=1}^T\|\z_t\|\right]$, which gives a data-dependent adaptive complexity. If $\E\left[\sum_{t=1}^T \|\z_t\|\right]\leq T^{\alpha}$ with $\alpha<1$, then the algorithm converges. Smaller $\alpha$ implies faster convergence. Our goal is to ensure that $\frac{1}{T}\sum_{t=1}^{T}\E\|\nabla F(\w_t)\|^2\leq \epsilon^2$. Choosing $b=1-\alpha$, $m=O(1)$ and $T_0=T^{1-\alpha}=O(\epsilon^{-2})$, and we end up with $T=O\left(\epsilon^{-\frac{2}{1-\alpha}}\right)$ complexity.

\item Theorem~\ref{thm:2} shows that in the ergodic sense, the Algorithm Adam$^+$ always converges, and the variance gets smaller when the number of iteration gets larger. Theorem~\ref{thm:2} rules out the case that the magnitude of $\z_t$ converges to a constant and the bound~(\ref{thm:eq4}) in Theorem~\ref{thm1} becomes vacuous.

\item To compare with Adam-style algorithms (e.g., Adam, AdaGrad), these algorithms' convergence depend on the growth rate of stochastic gradient, i.e., $\sum_{i=1}^d\|\g_{1:T,i}\|/T$, where $\g_{1:T,i}=[g_{1,i}, \ldots, g_{T, i}]$ denotes the $i$-th coordinate of all historical stochastic gradients. Hence, the data determines the growth rate of stochastic gradient. If the stochastic gradients are not sparse, then its growth rate may not be slow and these Adam-style algorithms may suffer from slow convergence. In contrast, for Adam$^+$ the convergence can be accelerated by the variance reduction property. Note that we have $ \E\left[\sum_{t=1}^{T}\|\z_t\|\right]/T\leq \E\left[\sum_{t=1}^{T}(\delta_t + \|\nabla F(\w_t)\|)\right]/T$. Hence, Adam$^+$'s convergence depends on the variance reduction property of $\z_t$. 

\end{itemize}

\subsection{A General Variant of Adam$^+$: Fast Convergence with Large Mini-batch}
Next, we introduce a more general variant of Adam$^+$ by making a simple change.  In particular, we keep all steps the same as in Algorithm~\ref{Alg:1} except the  adaptive step size is now set as $\eta_t =\frac{\alpha\beta^a}{\max\left(\left\|\z_t\right\|^{p}, \epsilon_0\right)}$, where $p\in[1/2,1)$ is parameter. We refer to this general variant of Adam$^+$ as power normalized Adam$^+$ (Nadam$^+$). This generalization allows us to compare with some existing methods and to establish fast convergence rate.  First, we notice that  when setting $p=1$ and $a=5/4$ and $\beta=1/T^{4/7}$,  Nadam$^+$ is almost the same as the stochastic method NIGT~\citep{cutkosky2020momentum} with only some minor differences. However, we observed that normalizing by $\|\z_t\|$ leads to slow convergence in practice, so we are instead interested in $p<1$. Below, we will show that NAdam$^+$ with $p<1$ can achieve a fast rate of $1/\epsilon^{3.5}$, which is the same as NIGT. 

\begin{thm}
	\label{thm3}
	Under the same assumption as in Theorem~\ref{thm1}, further assume $\sigma_0^2=\sigma^2/T_0$ and $\sigma_m^2=\sigma^2/m$. By using the step size $\eta_t=\frac{\alpha\beta^{4/3}}{\max\left(\left\|\z_t\right\|^{2/3}, \epsilon_0\right)}$ in Algorithm~\ref{Alg:1} with  $CL_H^2\alpha^4\leq 1/14$, $\alpha\leq 1/L$, $\epsilon_0=2\beta^{4/3}$, in order to have $\E\left[\|\nabla F(\w_\tau)\|\right]\leq\epsilon$ for a randomly selected solution $\w_{\tau}$ from $\{\w_1, \ldots, \w_T\}$, it suffice to set $\beta=O(\epsilon^{1/2})$, $T=O(\epsilon^{-2})$,  the initial batch size $T_0=1/\beta=O(\epsilon^{-1/2})$, the intermediate batch size as $m=1/\beta^3=O(\epsilon^{-3/2})$, which ends up with the total complexity $O(\epsilon^{-3.5})$. 

\end{thm}
\textbf{Remark:} Note that the above theorem establishes the fast convergence rate for Nadam$^+$ with $p=2/3$. Indeed, we can also establish a fast rate of Adam$^+$ (where $p=1/2$) in the order of $O(1/\epsilon^{3.625})$ with details provided in the Appendix~\ref{sec:newvariant}. 

\section{Experiments}
In this section, we conduct empirical studies to verify the effectiveness of the proposed algorithm on three different tasks: image classification on CIFAR10 and CIFAR100 dataset~\citep{krizhevsky2009cifar}, language modeling on Wiki-Text2 dataset~\citep{merity2016wikitext} and automatic speech recognition on SWB-300 dataset~\citep{saon-interspeech-2017}.
We choose tasks from different domains to demonstrate the applicability for the real-world deep learning tasks in a broad sense. The detailed description is presented in Table~\ref{table:1}. We compare our algorithm Adam$^+$ with SGD, momentum SGD, Adagrad, NIGT and Adam. We choose the same random initialization for each algorithm, and run a fixed number of epochs for every task. For Adam we choose the default setting $\beta_1=0.9$ and $\beta_2=0.999$ as in the original Adam paper.

\begin{table}[ht!]
	\centering
	\caption{Summary of setups in the experiments.}
	\label{table:1}
		\begin{tabular}{ccccc}
			\hline
			\textbf{Domain}	&   \textbf{Task}&\textbf{Architecture}       & \textbf{Dataset}  &    \\\hline
			Computer Vision	&	Image Classification & ResNet18        & CIFAR10  &    \\
			Computer Vision&	Image Classification &	VGG19                    & CIFAR100 &   \\
						Natural Language Processing& Language Modeling	&	 Two-layer LSTM      & Wiki-Text2  &   \\
			Automatic Speech Recognition&	 Speech Recognition &	Six-layer BiLSTM      & SWB-300  &   \\\hline
	\end{tabular}
\end{table}

\subsection{Image Classification}
\paragraph{CIFAR10 and CIFAR100} In the first experiment, we consider training ResNet18~\citep{he2016deep} and VGG19~\citep{simonyan2014very} to do image classification task on CIFAR10 and CIFAR100 dataset respectively.
For every optimizer, we use batch size 128 and run 350 epochs.
For SGD and momentum SGD, we set the initial learning rate to be $0.1$ for the first 150 epochs, and the learning rate is decreased by a factor of $10$ for every 100 epochs. For Adagrad and Adam, the initial learning rate is tuned from $\{0.1, 0.01, 0.001\}$ and we choose the one with the best performance. The best initial learning rates for Adagrad and Adam are $0.01$ and $0.001$ respectively. For NIGT, we tune the their momentum parameter from $\{0.01, 0.1, 0.9\}$ (the best momentum parameter we found is $0.9$) and the learning rate is chosen the same as in SGD. For Adam$^+$, the learning rate is set according to Algorithm~\ref{Alg:1}, in which we choose $\beta=0.1$ and the value of $\alpha$ is the same as the learning rate used in SGD. We report training and test accuracy versus the number of epochs in Figure~\ref{exp:cifar10} for CIFAR10 and Figure~\ref{exp:cifar100} for CIFAR100. 
We observe that our algorithm consistently outperforms all other algorithms on both CIFAR10 and CIFAR100, in terms of both training and testing accuracy. Notably, we have some interesting observations for the training of VGG19 on CIFAR100. First, both Adam$^+$ and NIGT significantly outperform SGD, momentum SGD, Adagrad and Adam.
Second, Adam$^+$ achieves almost the same final accuracy as NIGT, and Adam$^{+}$ converges much faster in the early stage of the training.

\begin{figure}[t]
	\centering
	\includegraphics[scale=0.3]{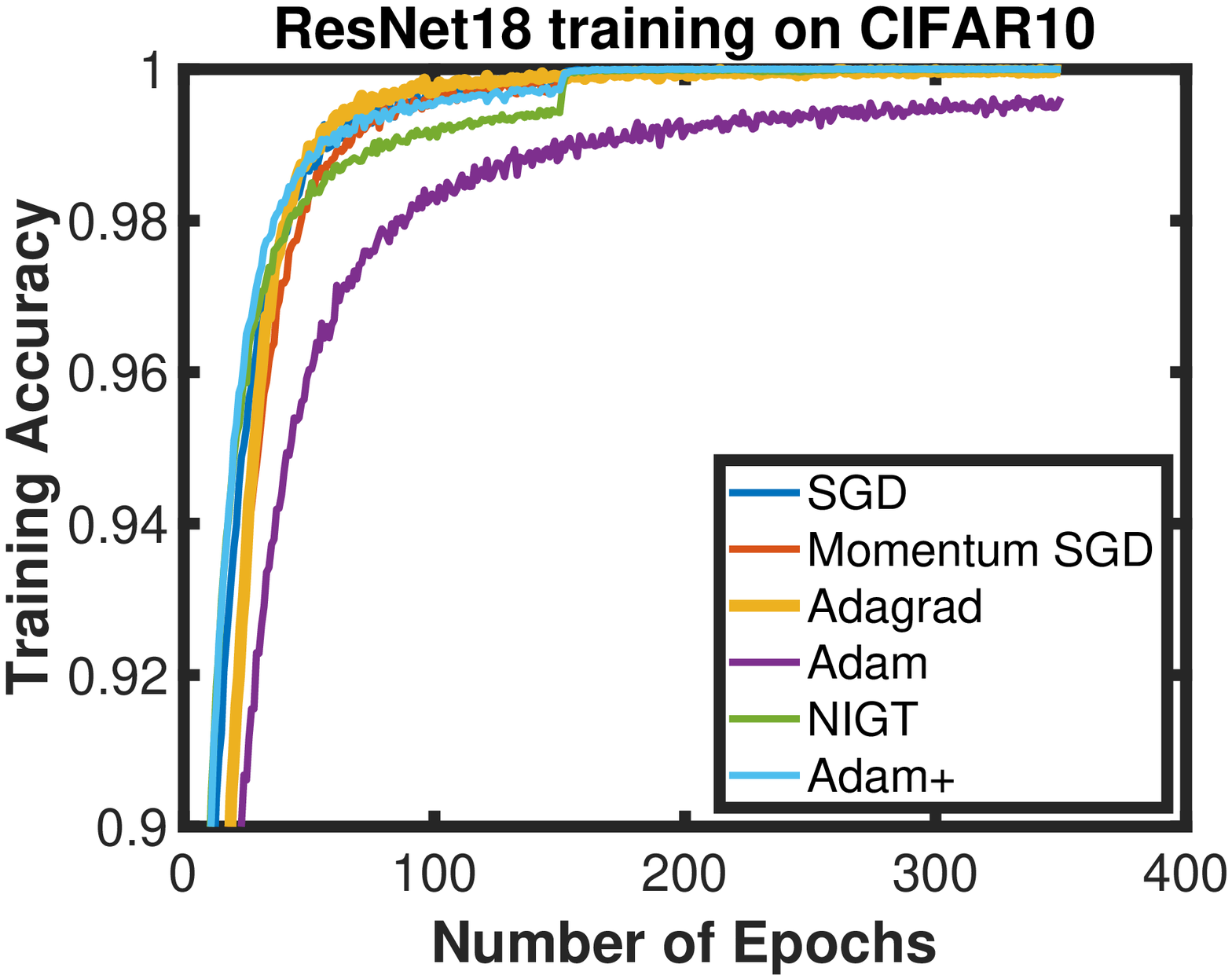}
	\includegraphics[scale=0.3]{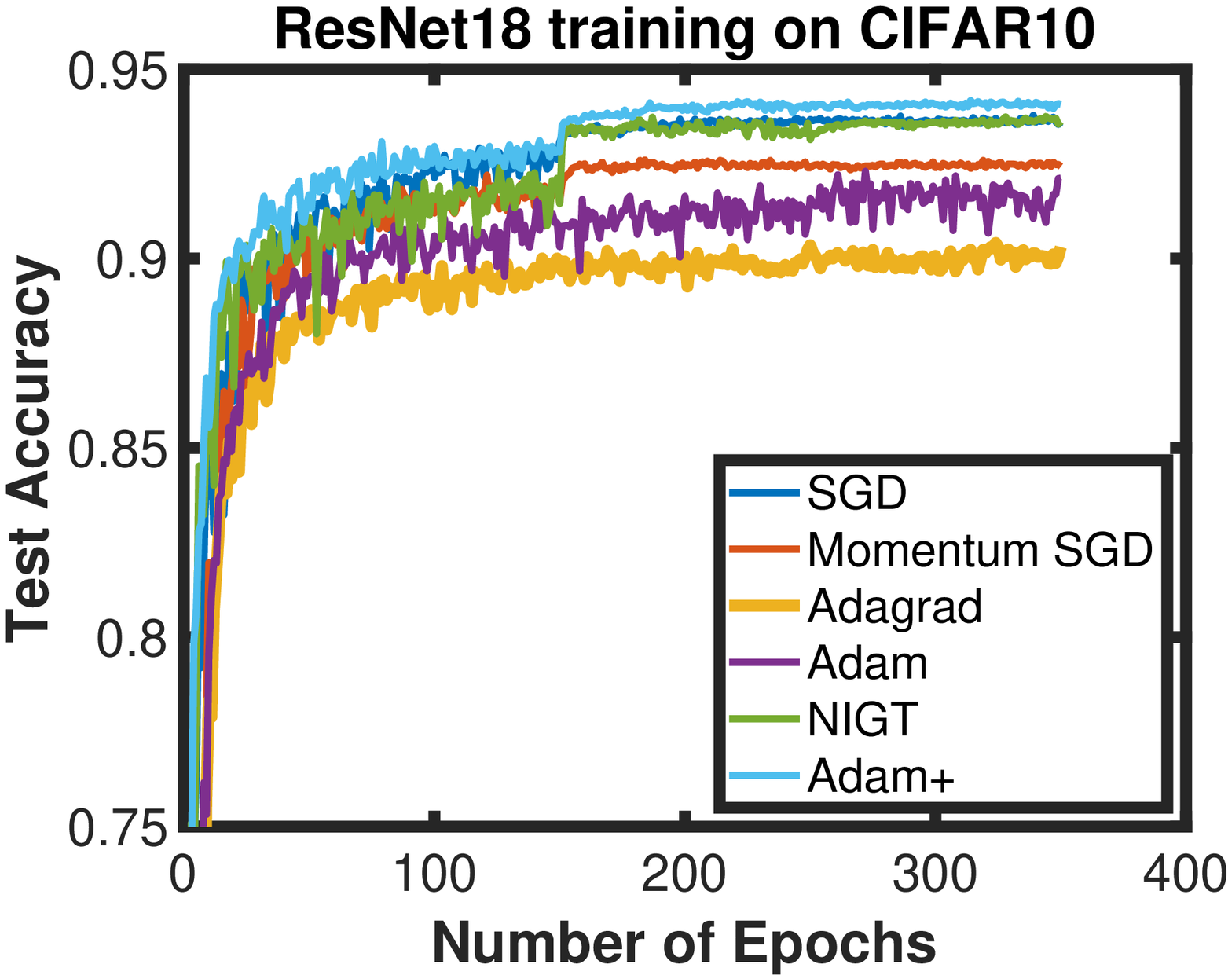}
	\caption{Comparison of optimization methods for ResNet18 Training on CIFAR10.}
	\label{exp:cifar10}
\end{figure}
\begin{figure}[t]
	\centering
	\includegraphics[scale=0.3]{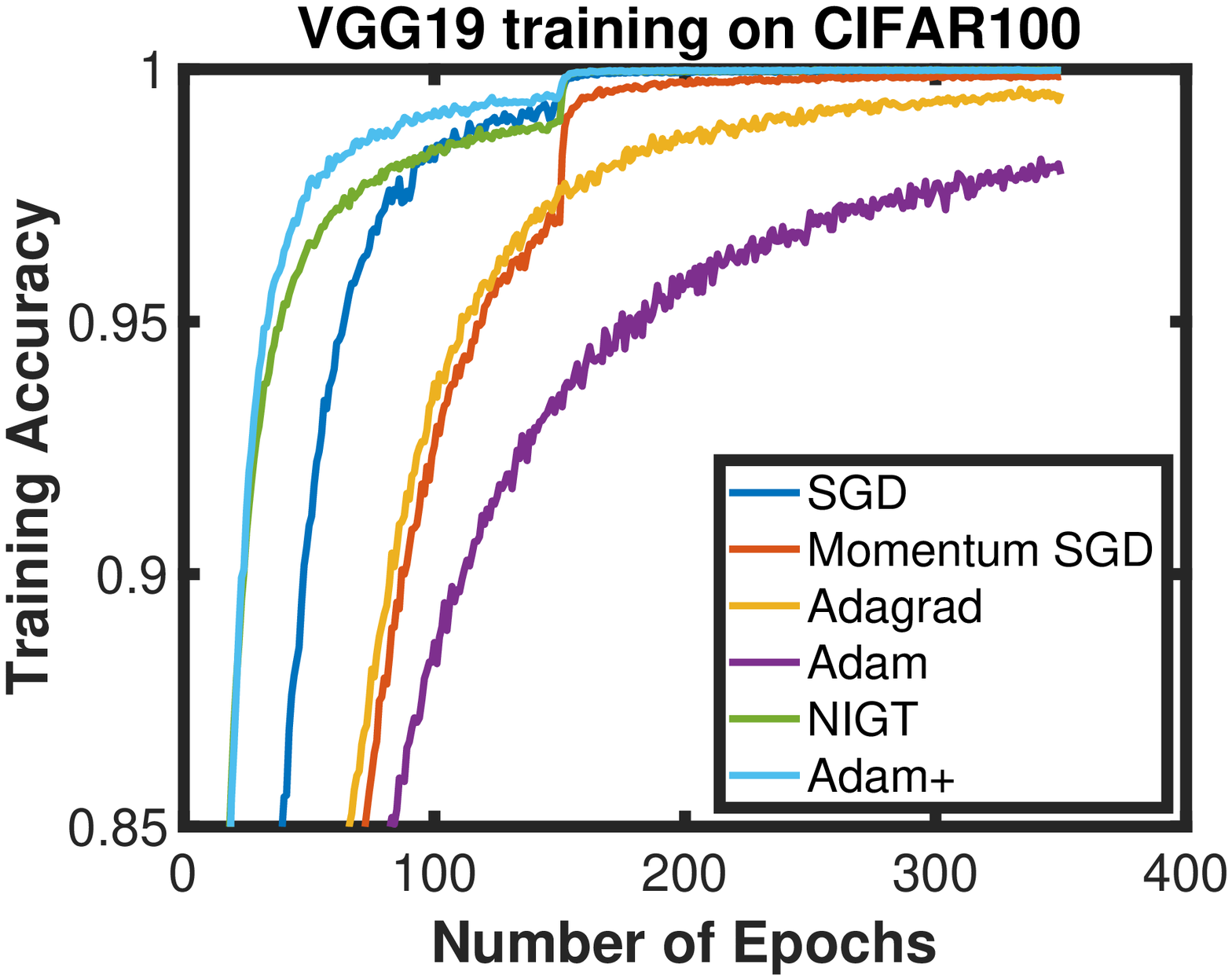}
	\includegraphics[scale=0.3]{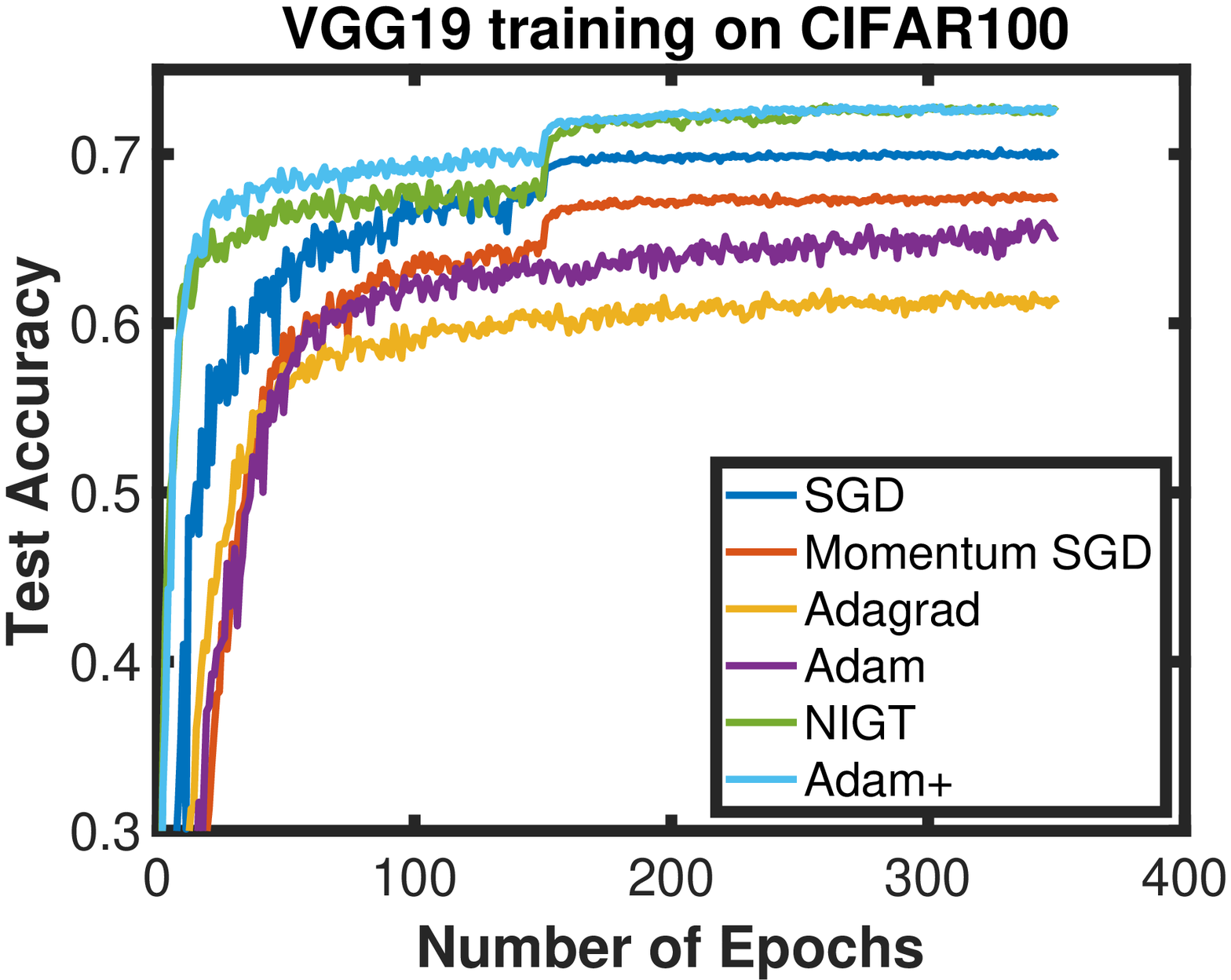}
	\caption{Comparison of optimization methods for VGG19 training on CIFAR100.}
	\label{exp:cifar100}
\end{figure}

\subsection{Language Modeling}

\paragraph{Wiki-text2} In the second experiment, we consider the language modeling task on WikiText-2 dataset. We use a 2-layer LSTM~\citep{hochreiter1997long}. The size of word embeddings is $650$ and the number of hidden units per layer is $650$. We run every algorithm for $40$ epochs, with batch size $20$ and dropout ratio $0.5$. For SGD and momentum SGD, we tune the initial learning rate from $\{0.1,0.2,0.5, 5,10,20\}$ and decrease the learning rate by factor of $4$ when the validation error saturates. For Adagrad and Adam, we tune the initial learning rate from $\{0.001, 0.01, 0.1, 1.0\}$. We report the best performance for these methods across the range of learning rate. The best initial learning rates for Adagrad and Adam are 0.01 and 0.001 respectively. For NIGT, we tune the initial value of learning rate from the same range as in SGD, and tune the momentum parameter $\beta$ from $\{0.01, 0.1, 0.9\}$, and the best parameter choice is $\beta=0.9$. The learning rate and $\beta$ are both decreased by a factor of $4$ when the validation error saturates. For Adam$^{+}$, we follow the same tuning strategy as NIGT.  


We report both training and test perplexity versus the number of epochs in Figure~\ref{exp:NLP}. From the Figure, we have the following observations: First, in terms of training perplexity, our algorithm achieves comparable performance with SGD and momentum SGD and outperforms Adagrad and NIGT, and it is worse than Adam. Second, in terms of test perplexity, our algorithm outperforms Adam, Adagrad, NIGT and momentum SGD, and it is comparable to SGD. An interesting observation is that Adam does not generalize well even if it has fast convergence in terms of training error, which is consistent with the observations in~\citep{wilson2017marginal}.
\begin{figure}[t]
	\centering
	\includegraphics[scale=0.3]{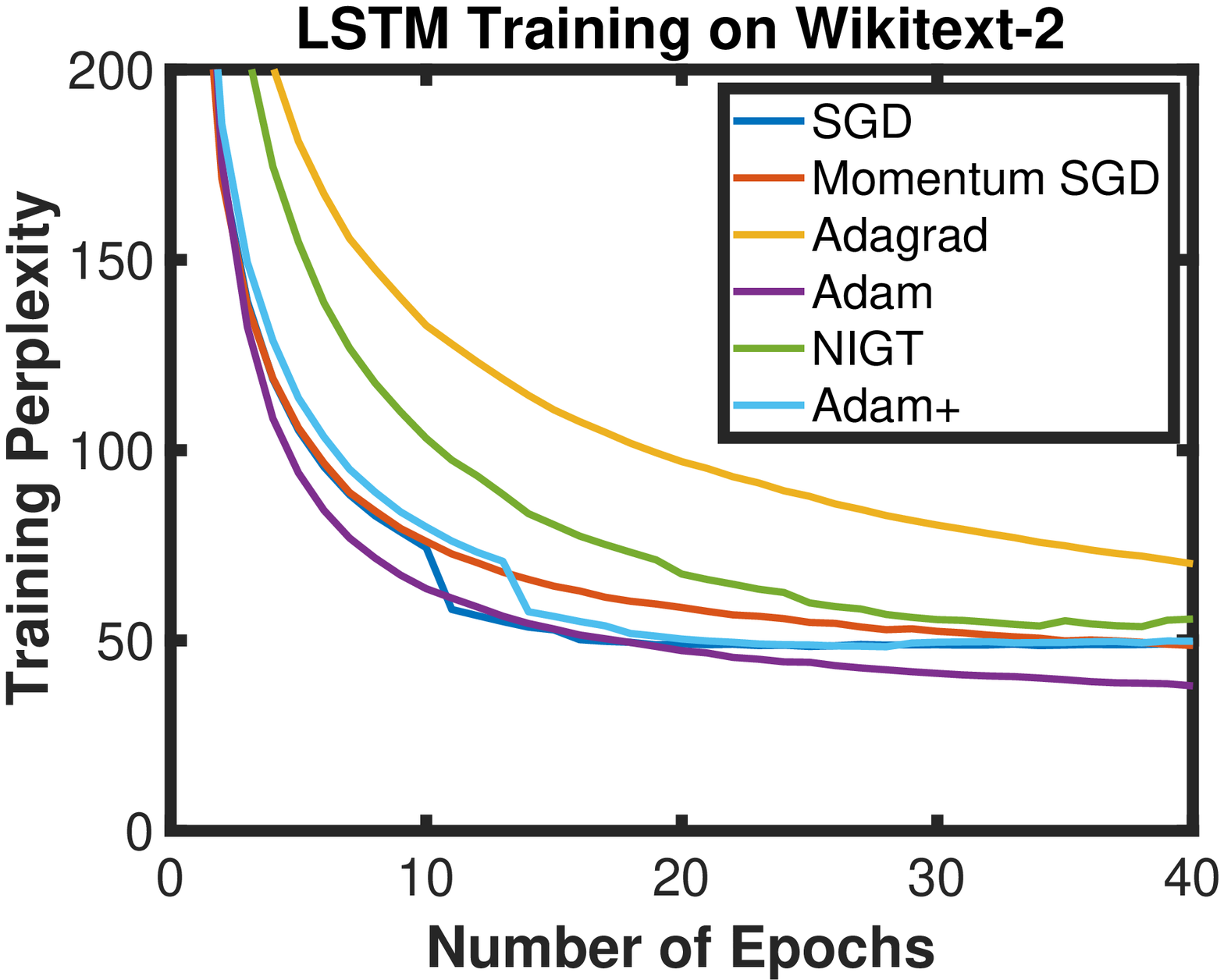}
	\includegraphics[scale=0.3]{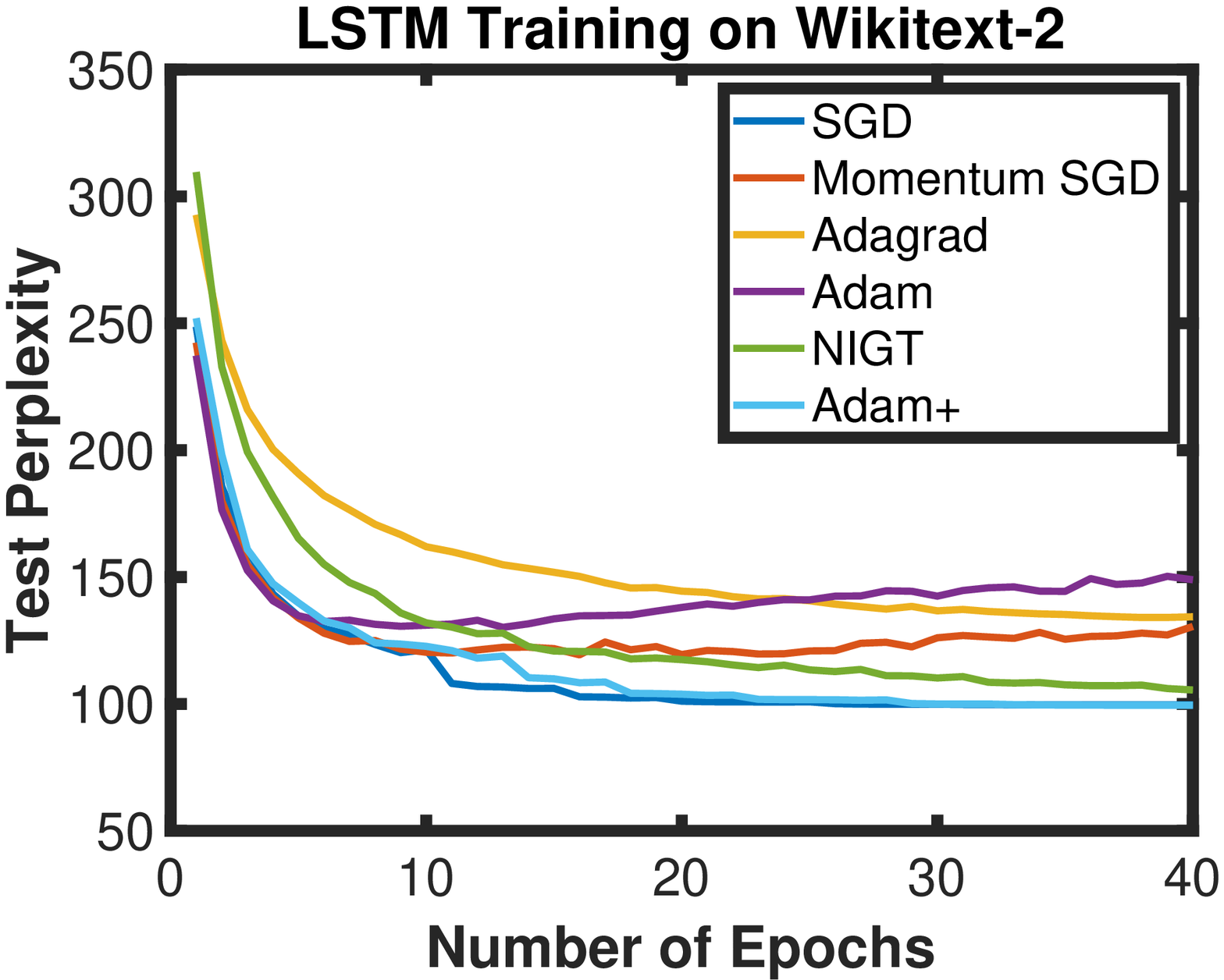}
	\caption{Comparison of optimization methods for two-layers LSTM training on WikiText-2.}
	\label{exp:NLP}
\end{figure}
\vspace*{-0.1in}
\subsection{Automatic Speech Recognition}
\paragraph{SWB-300} In the third experiment, we consider the automatic speech recognition task on SWB-300 dataset~\citep{saon-interspeech-2017}. SWB-300 contains roughly 300 hours of training data of over 4 million samples (30GB) and roughly 6 hours of held-out data of over 0.08 million samples (0.6GB). Each training sample is a fusion of FMLLR (40-dim), i-Vector (100-dim), and logmel with its delta and double delta. The acoustic model is a long short-term memory (LSTM) model with 6 bi-directional layers. Each layer contains 1,024 cells (512 cells in each direction). On top of the LSTM layers, there is a linear projection layer with 256 hidden units, followed by a softmax output layer with 32,000 (i.e., 32,000 classes) units corresponding to context-dependent HMM states. The LSTM is unrolled with 21 frames and trained with non-overlapping feature sub-sequences of that length. This model contains over 43 million parameters and is about 165MB large. The training takes about 20 hours on 1 V100 GPU. To compare, we adopt the well-tuned Momentum SGD strategy as described in \citep{icassp19} for this task as the baseline:  batch size is 256, learning rate is 0.1 for the first 10 epochs and then annealed by $\sqrt{0.5}$ for another 10 epochs, with momentum 0.9. We grid search the learning rate of Adam and Adagrad from $\{0.1, 0.01, 0.001\}$, and  report the best configuration we have found (Adam with learning rate $0.001$ and Adagrad with learning rate 0.01). For NIGT, we also follow the same learning rate setup (including annealing) as in Momentum SGD baseline. In addition, we fine tuned $\beta$ in NIGT by exploring $\beta$ in $\{0.01, 0.1, 0.9\}$ and reported the best configuration ($\beta=0.9$).  For Adam$^+$, we follow the same learning rate and annealing strategy as in the Momentum SGD  and tuned $\beta$ in the same way as in NIGT, reporting the best configuration ($\beta=0.01$). From Figure~\ref{exp:asr}, Adam$^+$ achieves the indistinguishable training loss and held-out loss w.r.t. well-tuned Momentum SGD baseline and significantly outperforms the other optimizers.


\begin{figure}[t]
	\centering
	\includegraphics[scale=0.3]{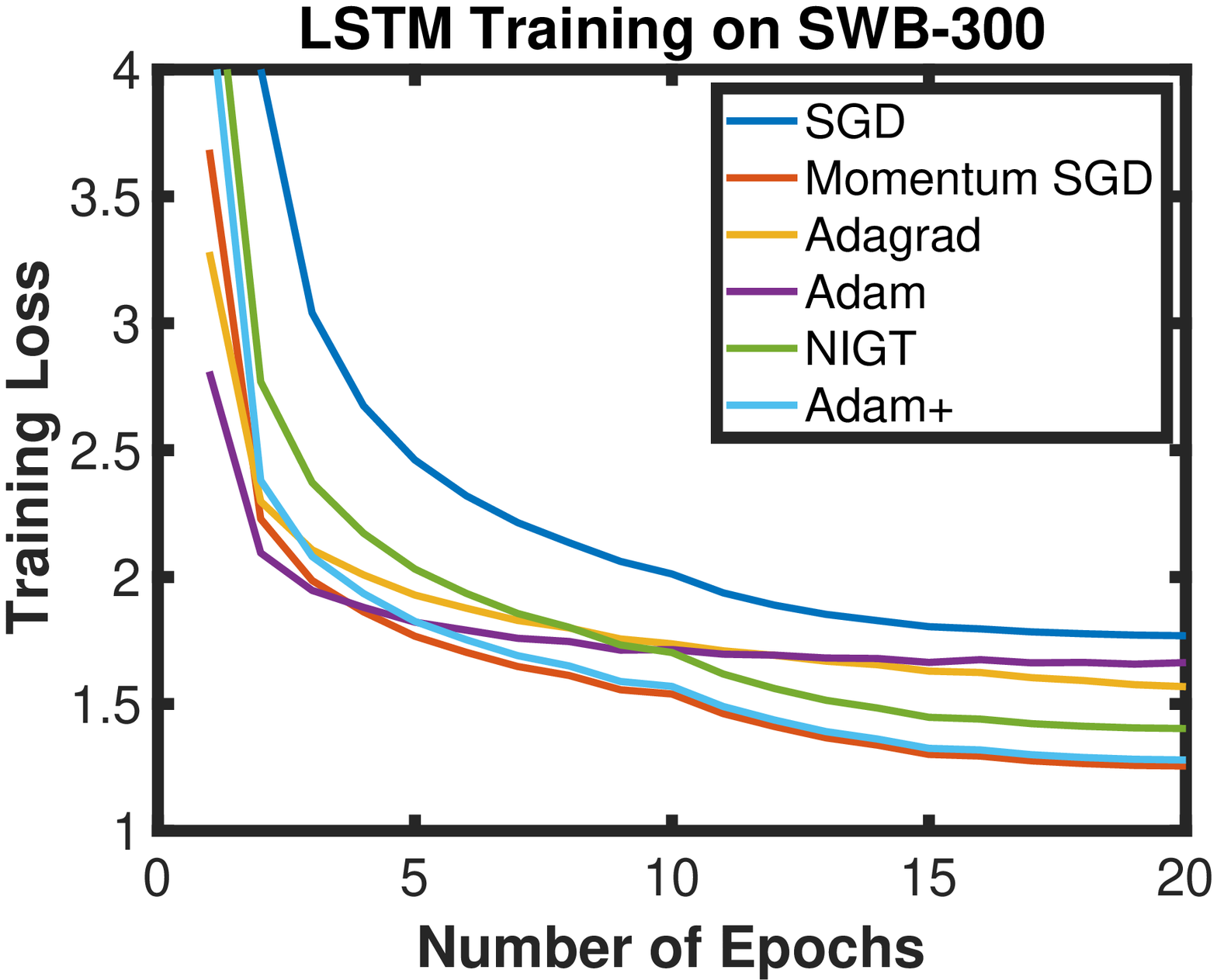}
	\includegraphics[scale=0.3]{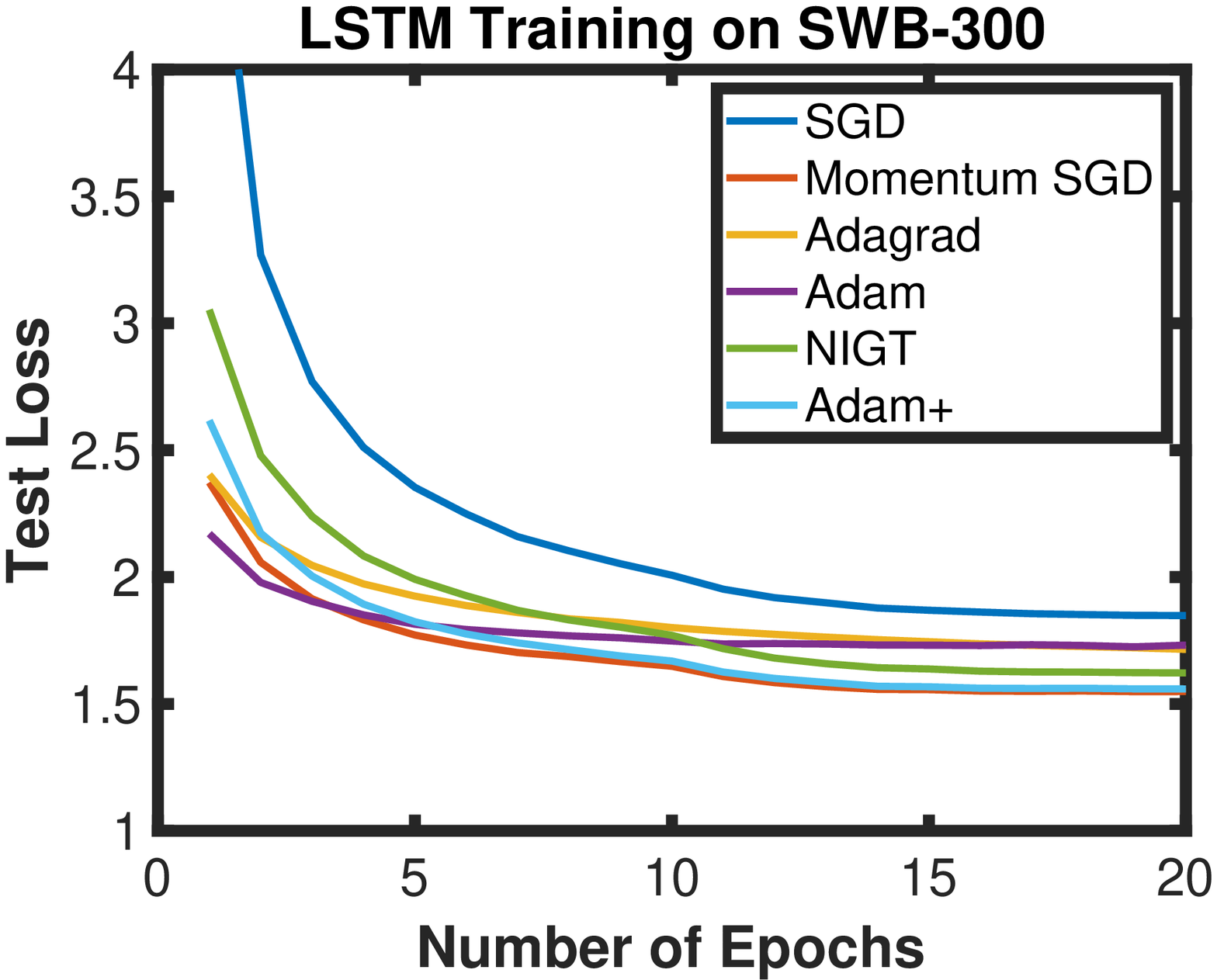}
	\caption{Comparison of optimization methods for six-layers LSTM training on SWB-300.}
	\label{exp:asr}
\end{figure}
\vspace*{-0.1in}
\subsection{Growth Rate of $\sum_{i=1}^{t}\|\z_i\|$}
\vspace*{-0.1in}
In this subsection, we consider the growth rate of  $\sum_{i=1}^{t}\|\z_i\|$, since they crucially affect the convergence rate as shown in Theorem~\ref{thm1}. We report the results of both ResNet18 training on CIFAR10 dataset and VGG19 training on CIFAR100 dataset. From Figure~\ref{exp:growth}, we can observe that it quickly reaches a plateau and then grows at a very slow rate with respect to the number of iterations. Specifically, the plateau is reached at around $6\times 10^4$ iterations, which corresponds to the epoch 154. At this particular epoch, the training and test accuracy are far from ideal and hence we need to keep the training until epoch 350. Then our algorithm Adam$^+$ is able to take advantage of the slow growth rate of $\sum_{i=1}^{t}\|\z_i\|$ for large $t$ and enjoys faster convergence, which is consistent with Figure~\ref{exp:cifar10}. This phenomenon verifies the variance reduction effect and also explains the reason why Adam$^{+}$ enjoys a fast convergence speed in practice.
\begin{figure}[t]
	\centering
	\includegraphics[scale=0.3]{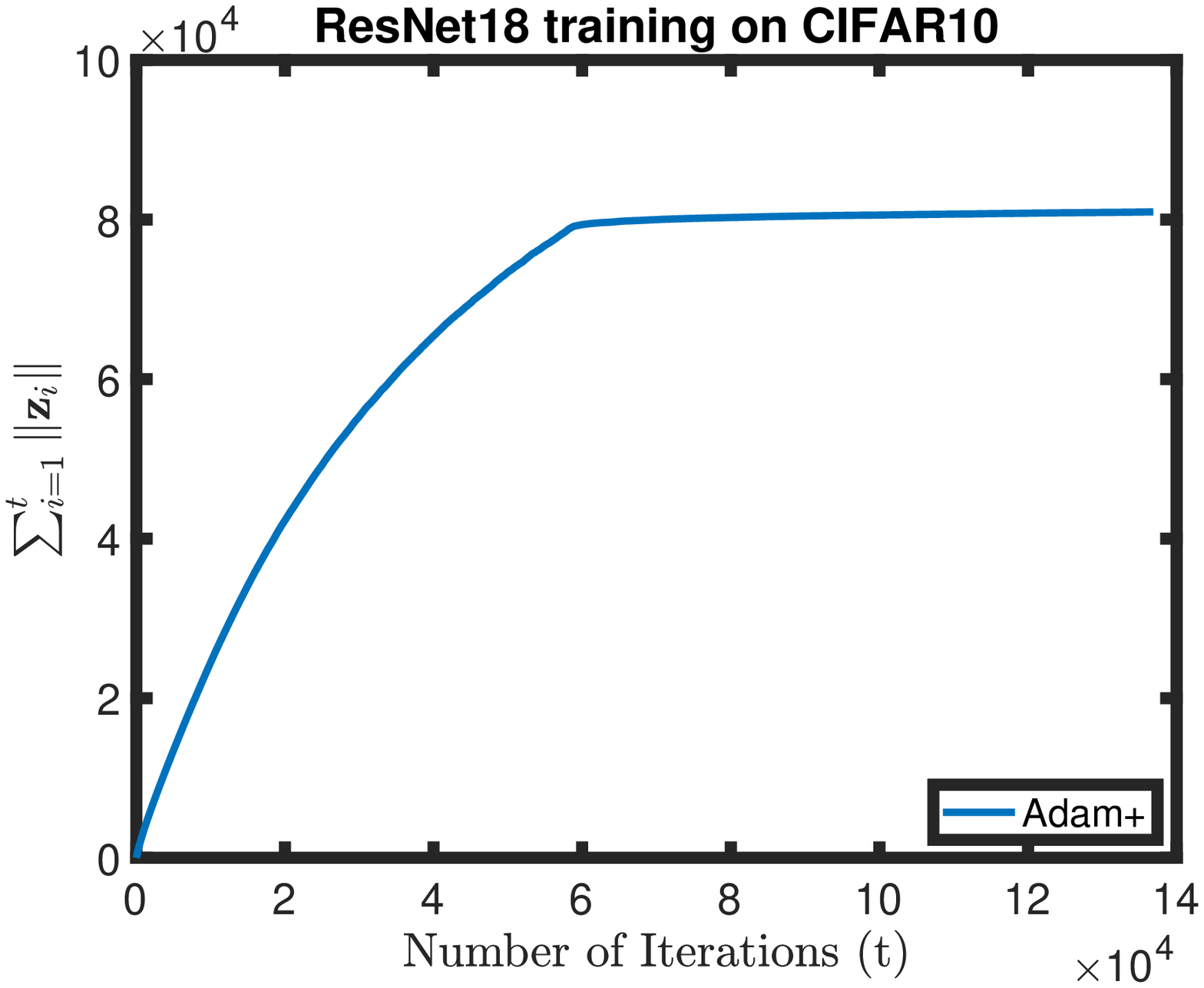}
	\includegraphics[scale=0.3]{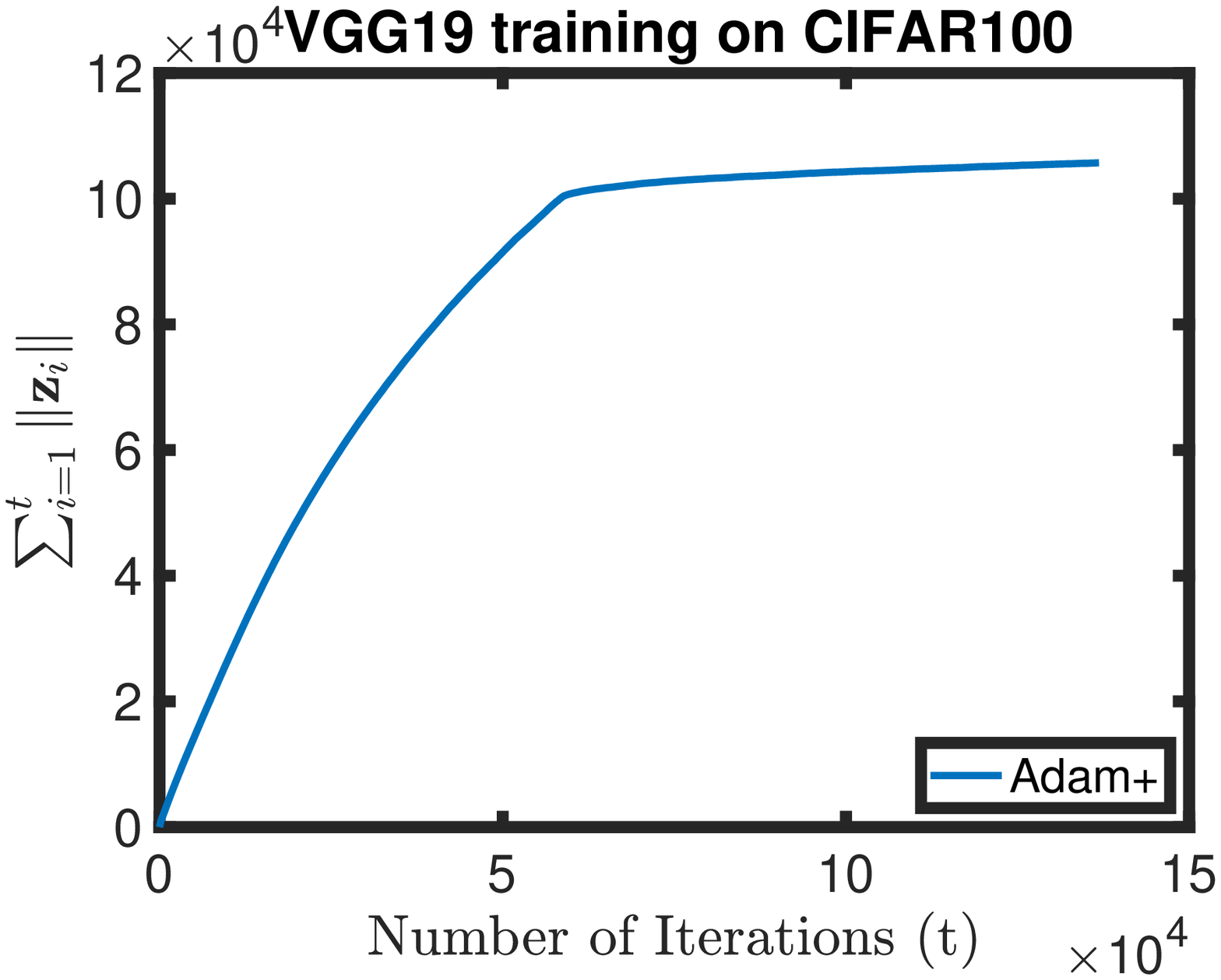}
	\caption{The growth of quantity $\sum_{i=1}^{t}\|\z_i\|$ in Adam$^+$}
	\label{exp:growth}
\end{figure}
\vspace*{-0.15in}
\section{Conclusion}
\vspace*{-0.15in}
In this paper, we design a new algorithm named Adam$^+$ to train deep neural networks efficiently. Different from Adam, Adam$^+$ updates the solution using moving average of stochastic gradients calculated at the extrapolated points and adaptive normalization on only first-order statistics of stochastic gradients. We establish data-dependent adaptive complexity results for Adam$^{+}$ from the perspective of adaptive variance reduction, and also show that a variant of Adam$^+$ achieves state-of-the-art complexity. Extensive empirical studies on several tasks verify the effectiveness of the proposed algorithm. We also empirically show that the slow growth rate of the new gradient estimator, providing the reason why Adam$^+$ enjoys fast convergence in practice.

\newpage
\bibliography{iclr2021_conference}

\begin{thebibliography}{42}
\providecommand{\natexlab}[1]{#1}
\providecommand{\url}[1]{\texttt{#1}}
\expandafter\ifx\csname urlstyle\endcsname\relax
  \providecommand{\doi}[1]{doi: #1}\else
  \providecommand{\doi}{doi: \begingroup \urlstyle{rm}\Url}\fi

\bibitem[Agarwal et~al.(2017)Agarwal, Allen-Zhu, Bullins, Hazan, and
  Ma]{agarwal2017finding}
Naman Agarwal, Zeyuan Allen-Zhu, Brian Bullins, Elad Hazan, and Tengyu Ma.
\newblock Finding approximate local minima faster than gradient descent.
\newblock In \emph{Proceedings of the 49th Annual ACM SIGACT Symposium on
  Theory of Computing}, pp.\  1195--1199, 2017.

\bibitem[Allen-Zhu \& Hazan(2016)Allen-Zhu and Hazan]{allen2016variance}
Zeyuan Allen-Zhu and Elad Hazan.
\newblock Variance reduction for faster non-convex optimization.
\newblock In \emph{International conference on machine learning}, pp.\
  699--707, 2016.

\bibitem[Arjevani et~al.(2019)Arjevani, Carmon, Duchi, Foster, Srebro, and
  Woodworth]{arjevani2019lower}
Yossi Arjevani, Yair Carmon, John~C Duchi, Dylan~J Foster, Nathan Srebro, and
  Blake Woodworth.
\newblock Lower bounds for non-convex stochastic optimization.
\newblock \emph{arXiv preprint arXiv:1912.02365}, 2019.

\bibitem[Carmon et~al.(2018)Carmon, Duchi, Hinder, and
  Sidford]{carmon2018accelerated}
Yair Carmon, John~C Duchi, Oliver Hinder, and Aaron Sidford.
\newblock Accelerated methods for nonconvex optimization.
\newblock \emph{SIAM Journal on Optimization}, 28\penalty0 (2):\penalty0
  1751--1772, 2018.

\bibitem[Chen et~al.(2018{\natexlab{a}})Chen, Zhou, Tang, Yang, and
  Gu]{chen2018closing}
Jinghui Chen, Dongruo Zhou, Yiqi Tang, Ziyan Yang, and Quanquan Gu.
\newblock Closing the generalization gap of adaptive gradient methods in
  training deep neural networks.
\newblock \emph{arXiv preprint arXiv:1806.06763}, 2018{\natexlab{a}}.

\bibitem[Chen et~al.(2018{\natexlab{b}})Chen, Liu, Sun, and
  Hong]{chen2018convergence}
Xiangyi Chen, Sijia Liu, Ruoyu Sun, and Mingyi Hong.
\newblock On the convergence of a class of {Adam}-type algorithms for
  non-convex optimization.
\newblock \emph{arXiv preprint arXiv:1808.02941}, 2018{\natexlab{b}}.

\bibitem[Chen et~al.(2019)Chen, Yuan, Yi, Zhou, Chen, and
  Yang]{chen2018universal}
Zaiyi Chen, Zhuoning Yuan, Jinfeng Yi, Bowen Zhou, Enhong Chen, and Tianbao
  Yang.
\newblock Universal stagewise learning for non-convex problems with convergence
  on averaged solutions.
\newblock In \emph{International Conference on Learning Representations}, 2019.

\bibitem[Cutkosky \& Mehta(2020)Cutkosky and Mehta]{cutkosky2020momentum}
Ashok Cutkosky and Harsh Mehta.
\newblock Momentum improves normalized {SGD}.
\newblock \emph{arXiv preprint arXiv:2002.03305}, 2020.

\bibitem[Cutkosky \& Orabona(2019)Cutkosky and Orabona]{cutkosky2019momentum}
Ashok Cutkosky and Francesco Orabona.
\newblock Momentum-based variance reduction in non-convex {SGD}.
\newblock In \emph{Advances in Neural Information Processing Systems}, pp.\
  15236--15245, 2019.

\bibitem[Duchi et~al.(2011)Duchi, Hazan, and Singer]{duchi2011adaptive}
John Duchi, Elad Hazan, and Yoram Singer.
\newblock Adaptive subgradient methods for online learning and stochastic
  optimization.
\newblock \emph{Journal of Machine Learning Research}, 12\penalty0
  (Jul):\penalty0 2121--2159, 2011.

\bibitem[Fang et~al.(2018)Fang, Li, Lin, and Zhang]{fang2018spider}
Cong Fang, Chris~Junchi Li, Zhouchen Lin, and Tong Zhang.
\newblock Spider: Near-optimal non-convex optimization via stochastic
  path-integrated differential estimator.
\newblock In \emph{Advances in Neural Information Processing Systems}, pp.\
  689--699, 2018.

\bibitem[Fang et~al.(2019)Fang, Lin, and Zhang]{fang2019sharp}
Cong Fang, Zhouchen Lin, and Tong Zhang.
\newblock Sharp analysis for nonconvex {SGD} escaping from saddle points.
\newblock \emph{arXiv preprint arXiv:1902.00247}, 2019.

\bibitem[Ghadimi \& Lan(2013)Ghadimi and Lan]{ghadimi2013stochastic}
Saeed Ghadimi and Guanghui Lan.
\newblock Stochastic first-and zeroth-order methods for nonconvex stochastic
  programming.
\newblock \emph{SIAM Journal on Optimization}, 23\penalty0 (4):\penalty0
  2341--2368, 2013.

\bibitem[Goyal et~al.(2017)Goyal, Doll{\'a}r, Girshick, Noordhuis, Wesolowski,
  Kyrola, Tulloch, Jia, and He]{goyal2017accurate}
Priya Goyal, Piotr Doll{\'a}r, Ross Girshick, Pieter Noordhuis, Lukasz
  Wesolowski, Aapo Kyrola, Andrew Tulloch, Yangqing Jia, and Kaiming He.
\newblock Accurate, large minibatch {SGD}: Training imagenet in 1 hour.
\newblock \emph{arXiv preprint arXiv:1706.02677}, 2017.

\bibitem[He et~al.(2016)He, Zhang, Ren, and Sun]{he2016deep}
Kaiming He, Xiangyu Zhang, Shaoqing Ren, and Jian Sun.
\newblock Deep residual learning for image recognition.
\newblock In \emph{Proceedings of the IEEE conference on computer vision and
  pattern recognition}, pp.\  770--778, 2016.

\bibitem[Hochreiter \& Schmidhuber(1997)Hochreiter and
  Schmidhuber]{hochreiter1997long}
Sepp Hochreiter and J{\"u}rgen Schmidhuber.
\newblock Long short-term memory.
\newblock \emph{Neural computation}, 9\penalty0 (8):\penalty0 1735--1780, 1997.

\bibitem[Jin et~al.(2017)Jin, Ge, Netrapalli, Kakade, and
  Jordan]{jin2017escape}
Chi Jin, Rong Ge, Praneeth Netrapalli, Sham~M Kakade, and Michael~I Jordan.
\newblock How to escape saddle points efficiently.
\newblock \emph{arXiv preprint arXiv:1703.00887}, 2017.

\bibitem[Johnson \& Zhang(2013)Johnson and Zhang]{johnson2013accelerating}
Rie Johnson and Tong Zhang.
\newblock Accelerating stochastic gradient descent using predictive variance
  reduction.
\newblock In \emph{Advances in neural information processing systems}, pp.\
  315--323, 2013.

\bibitem[Kingma \& Ba(2014)Kingma and Ba]{kingma2014adam}
Diederik~P Kingma and Jimmy Ba.
\newblock Adam: A method for stochastic optimization.
\newblock \emph{arXiv preprint arXiv:1412.6980}, 2014.

\bibitem[Krizhevsky et~al.(2009)Krizhevsky, Nair, and
  Hinton]{krizhevsky2009cifar}
Alex Krizhevsky, Vinod Nair, and Geoffrey Hinton.
\newblock {CIFAR}-10 and {CIFAR}-100 datasets.
\newblock \emph{URl: https://www. cs. toronto. edu/kriz/cifar. html},
  6:\penalty0 1, 2009.

\bibitem[Lei et~al.(2017)Lei, Ju, Chen, and Jordan]{lei2017non}
Lihua Lei, Cheng Ju, Jianbo Chen, and Michael~I Jordan.
\newblock Non-convex finite-sum optimization via {SCSG} methods.
\newblock In \emph{Advances in Neural Information Processing Systems}, pp.\
  2348--2358, 2017.

\bibitem[Levy(2017)]{levy2017online}
Kfir Levy.
\newblock Online to offline conversions, universality and adaptive minibatch
  sizes.
\newblock In \emph{Advances in Neural Information Processing Systems}, pp.\
  1613--1622, 2017.

\bibitem[Li \& Orabona(2019)Li and Orabona]{li2019convergence}
Xiaoyu Li and Francesco Orabona.
\newblock On the convergence of stochastic gradient descent with adaptive
  stepsizes.
\newblock In \emph{The 22nd International Conference on Artificial Intelligence
  and Statistics}, pp.\  983--992, 2019.

\bibitem[Liu et~al.(2019)Liu, Jiang, He, Chen, Liu, Gao, and
  Han]{liu2019variance}
Liyuan Liu, Haoming Jiang, Pengcheng He, Weizhu Chen, Xiaodong Liu, Jianfeng
  Gao, and Jiawei Han.
\newblock On the variance of the adaptive learning rate and beyond.
\newblock \emph{arXiv preprint arXiv:1908.03265}, 2019.

\bibitem[Luo et~al.(2019)Luo, Xiong, Liu, and Sun]{luo2019adaptive}
Liangchen Luo, Yuanhao Xiong, Yan Liu, and Xu~Sun.
\newblock Adaptive gradient methods with dynamic bound of learning rate.
\newblock \emph{arXiv preprint arXiv:1902.09843}, 2019.

\bibitem[McMahan \& Streeter(2010)McMahan and Streeter]{mcmahan2010adaptive}
H~Brendan McMahan and Matthew Streeter.
\newblock Adaptive bound optimization for online convex optimization.
\newblock \emph{arXiv preprint arXiv:1002.4908}, 2010.

\bibitem[Merity(2016)]{merity2016wikitext}
Stephen Merity.
\newblock The {WikiText} long term dependency language modeling dataset.
\newblock \emph{Salesforce Metamind}, 9, 2016.

\bibitem[Pham et~al.(2020)Pham, Nguyen, Phan, and Tran-Dinh]{pham2020proxsarah}
Nhan~H Pham, Lam~M Nguyen, Dzung~T Phan, and Quoc Tran-Dinh.
\newblock {ProxSARAH}: An efficient algorithmic framework for stochastic
  composite nonconvex optimization.
\newblock \emph{Journal of Machine Learning Research}, 21\penalty0
  (110):\penalty0 1--48, 2020.

\bibitem[Reddi et~al.(2016)Reddi, Hefny, Sra, Poczos, and
  Smola]{reddi2016stochastic}
Sashank~J Reddi, Ahmed Hefny, Suvrit Sra, Barnabas Poczos, and Alex Smola.
\newblock Stochastic variance reduction for nonconvex optimization.
\newblock In \emph{International conference on machine learning}, pp.\
  314--323, 2016.

\bibitem[Reddi et~al.(2019)Reddi, Kale, and Kumar]{reddi2019convergence}
Sashank~J Reddi, Satyen Kale, and Sanjiv Kumar.
\newblock On the convergence of {Adam} and beyond.
\newblock \emph{arXiv preprint arXiv:1904.09237}, 2019.

\bibitem[Saon et~al.(2017)Saon, Kurata, Sercu, Audhkhasi, Thomas, Dimitriadis,
  Cui, Ramabhadran, Picheny, Lim, Roomi, and Hall]{saon-interspeech-2017}
George Saon, Gakuto Kurata, Tom Sercu, Kartik Audhkhasi, Samuel Thomas,
  Dimitrios Dimitriadis, Xiaodong Cui, Bhuvana Ramabhadran, Michael Picheny,
  Lynn-Li Lim, Bergul Roomi, and Phil Hall.
\newblock English conversational telephone speech recognition by humans and
  machines.
\newblock In \emph{Interspeech}, 2017.

\bibitem[Simonyan \& Zisserman(2014)Simonyan and Zisserman]{simonyan2014very}
Karen Simonyan and Andrew Zisserman.
\newblock Very deep convolutional networks for large-scale image recognition.
\newblock \emph{arXiv preprint arXiv:1409.1556}, 2014.

\bibitem[Tieleman \& Hinton(2012)Tieleman and Hinton]{tieleman2012lecture}
Tijmen Tieleman and Geoffrey Hinton.
\newblock Lecture 6.5-rmsprop, coursera: Neural networks for machine learning.
\newblock \emph{University of Toronto, Technical Report}, 2012.

\bibitem[Wang et~al.(2017)Wang, Fang, and Liu]{wang2017stochastic}
Mengdi Wang, Ethan~X Fang, and Han Liu.
\newblock Stochastic compositional gradient descent: algorithms for minimizing
  compositions of expected-value functions.
\newblock \emph{Mathematical Programming}, 161\penalty0 (1-2):\penalty0
  419--449, 2017.

\bibitem[Wang et~al.(2019)Wang, Ji, Zhou, Liang, and
  Tarokh]{wang2019spiderboost}
Zhe Wang, Kaiyi Ji, Yi~Zhou, Yingbin Liang, and Vahid Tarokh.
\newblock {SpiderBoost} and momentum: Faster variance reduction algorithms.
\newblock In \emph{Advances in Neural Information Processing Systems}, pp.\
  2406--2416, 2019.

\bibitem[Ward et~al.(2019)Ward, Wu, and Bottou]{ward2019adagrad}
Rachel Ward, Xiaoxia Wu, and Leon Bottou.
\newblock {AdaGrad} stepsizes: Sharp convergence over nonconvex landscapes.
\newblock In \emph{International Conference on Machine Learning}, pp.\
  6677--6686, 2019.

\bibitem[Wilson et~al.(2017)Wilson, Roelofs, Stern, Srebro, and
  Recht]{wilson2017marginal}
Ashia~C Wilson, Rebecca Roelofs, Mitchell Stern, Nati Srebro, and Benjamin
  Recht.
\newblock The marginal value of adaptive gradient methods in machine learning.
\newblock In \emph{Advances in Neural Information Processing Systems}, pp.\
  4148--4158, 2017.

\bibitem[You et~al.(2017)You, Gitman, and Ginsburg]{you2017scaling}
Yang You, Igor Gitman, and Boris Ginsburg.
\newblock Scaling {SGD} batch size to 32k for imagenet training.
\newblock \emph{arXiv preprint arXiv:1708.03888}, 6, 2017.

\bibitem[You et~al.(2019)You, Li, Reddi, Hseu, Kumar, Bhojanapalli, Song,
  Demmel, Keutzer, and Hsieh]{you2019large}
Yang You, Jing Li, Sashank Reddi, Jonathan Hseu, Sanjiv Kumar, Srinadh
  Bhojanapalli, Xiaodan Song, James Demmel, Kurt Keutzer, and Cho-Jui Hsieh.
\newblock Large batch optimization for deep learning: Training {BERT} in 76
  minutes.
\newblock \emph{arXiv preprint arXiv:1904.00962}, 2019.

\bibitem[Zhang et~al.(2020)Zhang, Lin, Sra, and Jadbabaie]{zhang2020complexity}
Jingzhao Zhang, Hongzhou Lin, Suvrit Sra, and Ali Jadbabaie.
\newblock On complexity of finding stationary points of nonsmooth nonconvex
  functions.
\newblock \emph{arXiv preprint arXiv:2002.04130}, 2020.

\bibitem[Zhang et~al.(2019)Zhang, Cui, Finkler, Kingsbury, Saon, Kung, and
  Picheny]{icassp19}
Wei Zhang, Xiaodong Cui, Ulrich Finkler, Brian Kingsbury, George Saon, David
  Kung, and Michael Picheny.
\newblock Distributed deep learning strategies for automatic speech
  recognition.
\newblock In \emph{ICASSP'2019}, May 2019.

\bibitem[Zhou et~al.(2018)Zhou, Xu, and Gu]{zhou2018stochastic}
Dongruo Zhou, Pan Xu, and Quanquan Gu.
\newblock Stochastic nested variance reduction for nonconvex optimization.
\newblock In \emph{Advances in Neural Information Processing Systems}, pp.\
  3921--3932, 2018.

\end{thebibliography}
\bibliographystyle{iclr2021_conference}
\newpage
\appendix
\section{Proof of Lemma~\ref{lemma:1}}
\label{proof:lemma1}
\begin{proof}
	The proof is similar to that of Lemma 12 in~\citep{wang2017stochastic}. Define
	\begin{equation}
	\begin{aligned}
	\zeta_k^{(t)}=\begin{cases}
	\beta (1-\beta)^{t-k} &  \text{if  } t\geq k> 0\\
	(1-\beta)^{t-k}					& \text{if } t\geq k=0
	\end{cases}\\
	\end{aligned}
	\end{equation}
	By the definition of $\zeta_t^{(k)}$ and the update of Algorithm~\ref{Alg:1}, we have 
	$$\zeta_k^{(t+1)}=(1-\beta)\zeta_k^{(t)}, \quad \sum_{k=0}^{t}\zeta_k^{(t)}=1,\quad \w_t=\sum_{k=0}^{t}\zeta_k^{(t)}\wh_{t+1},\quad \z_{t+1}=\sum_{k=0}^{t}\zeta_k^{(t)}\nabla f(\wh_{t+1};\xi_{t+1}).$$
	Define $m_{t+1}=\sum_{k=0}^{t}\zeta_k^{(t)}\|\w_{t+1}-\wh_{k+1}\|^2$, $n_{t+1}=\sum_{k=0}^{t}\zeta_k^{(t)}\left[\nabla f(\wh_{k+1};\xi_{k+1})-F(\wh_{k+1})\right]$, where $\nabla f(\wh_{k+1};\xi_{k+1})$ is an unbiased stochastic first-order oracle for $F(\wh_{k+1})$ with bounded variance $\sigma_m^2$.
	Note that $\nabla F$ is a $L_H$-smooth mapping (according to Assumption~\ref{ass:1} (iii)), then by Lemma 10 of~\citep{wang2017stochastic}, we have
	\begin{equation*}
	\|\z_t-\nabla F(\w_t)\|^2\leq (L_Hm_t+\|n_t\|)^2\leq 2L_{H}^2m_t^2+2\|n_t\|^2.
	\end{equation*}
	Define $q_{t+1}=\sum_{k=0}^{t}\zeta_k^{(t)}\left\|\w_{t+1}-\wh_{k+1}\right\|$. According to Lemma 11 (a) and (b) of~\citep{wang2017stochastic}, we have
	\begin{equation*}
	m_{t+1}+4q_{t+1}^2\leq \left(1-\frac{\beta}{2}\right)\left(m_t+4q_t^2\right)+\frac{18}{\beta}\|\w_{t+1}-\w_t\|^2.
	\end{equation*}
	Taking squares on both sides of the inequality and using the fact that $(a+b)^2\leq (1+\frac{\beta}{2})a^2+(1+\frac{2}{\beta})b^2$ for $\beta>0$, we have
	\begin{equation}
	\label{inequality:1}
	\begin{aligned}
	\left(m_{t+1}+4q_{t+1}^2\right)^2&\leq \left(1+\frac{\beta}{2}\right)\left(1-\frac{\beta}{2}\right)^2\left(m_t+4q_t^2\right)^2+\left(1+\frac{2}{\beta}\right)\frac{324}{\beta^2}\|\w_{t+1}-\w_t\|^4\\
	&\leq \left(1-\frac{\beta}{2}\right)(m_t+4q_t^2)^2+\frac{972}{\beta^3}\left\|\w_{t+1}-\w_t\right\|^4,
	\end{aligned}
	\end{equation}
	where the last inequality holds since $1/\beta\geq 1$.

	Define $\delta_t^2=2L_{H}^2(m_t+4q_t^2)^2+2\left\|n_t\right\|^2$, then we have $\left\|\z_t-\nabla F(\w_t)\right\|^2\leq \delta_t^2$ for all $t$. Denote $\mathcal{F}_{t+1}$ by the $\sigma$-algebra generated by $\xi_1,\ldots,\xi_{t+1}$. Taking the summation of (\ref{inequality:1}) and according to the bound of $n_t$ derived in Lemma 11 (c) of~\citep{wang2017stochastic}, we have
	\begin{equation*}
	\begin{aligned}
	\E\left[\delta_{t+1}^2\vert \mathcal{F}_{t+1}\right]\leq \left(1-\frac{\beta}{2}\right)\delta_t^2+2\beta^2\sigma_m^2+\frac{1944L_{H}^2\left\|\w_{t+1}-\w_t\right\|^4}{\beta^3},
	\end{aligned}
	\end{equation*}
	Taking expectation on both sides yields
	\begin{align*}
	\E\left[\delta_{t+1}^2\right] \leq \left(1- \frac{\beta}{2}\right)\E\left[\delta_t^2\right] +2\beta^2 \sigma_m^2 + \E\left[\frac{1944L_{H}^2\|\w_{t+1} - \w_t\|^4}{\beta^3}\right].
	\end{align*}
	Note that $\eta_t=\frac{\alpha\beta^a}{\max\left(\|\z_t\|^{1/2},\epsilon_0\right)}$, we have
	\begin{equation*}
	\E\left[\delta_{t+1}^2\right] \leq    \left(1- \frac{\beta}{2}\right)\E\left[\delta_t^2\right] +2\beta^2 \sigma_m^2 + \E\left[CL^2\alpha^4\beta^{4a-3}\|\z_t\|^2\right]. \qedhere
	\end{equation*}
\end{proof}
\section{Proof of Theorem~\ref{thm1}}
\begin{proof}
	By Lemma~\ref{lemma:1} and the update rule of Algorithm~\ref{Alg:1}, we have
	\begin{equation}
	\label{mainproof:ineq1}
	\begin{aligned}
	\E\left[\delta_{t+1}^2\right] &\leq \left(1- \frac{\beta}{2}\right)\E\left[\delta_t^2\right] +2\beta^2 \sigma_m^2 + \E\left[\frac{CL_H^2\alpha^4\beta^{4a}\|\z_t\|^4}{\beta^3(\max(\|\z_t\|^{1/2},\epsilon_0))^4}\right]\\
	& \leq \left(1- \frac{\beta}{2}\right)\E\left[\delta_t^2\right] +2\beta^2 \sigma_m^2 + \E\left[\frac{2CL_H^2\alpha^4\beta^4(\|\delta_t\|^2 + \|\nabla F(\w_t)\|^2)}{\beta^3}\right],
	\end{aligned}
	\end{equation}
	where the second inequality holds since $(\max(\|\z_t\|^{1/2},\epsilon_0))^4\geq \|\z_t\|^2$ and $\|\z_t\|^2\leq 2\|\delta_t\|^2+2\left\|\nabla F(\w_t)\right\|^2$.
	
	Note that $2CL_H^2\alpha^4\leq 1/18$. Plugging it in~(\ref{mainproof:ineq1}), we have
	\begin{equation}
	\label{mainproof:ineq2}
	\frac{8\beta}{18}\E\left[\delta_t^2\right]\leq \E\left[\delta_t^2 -\delta_{t+1}^2 \right] +2\beta^2 \sigma_m^2 +  \E\left[\frac{\beta}{18}\|\nabla F(\w_t)\|^2\right].
	\end{equation}
	Summing over $t=1,\ldots,T$ on both sides of~(\ref{mainproof:ineq2}) and with some simple algebra, we have 
	\begin{equation}
	\label{mainproof:ineq3}
	\sum_{t=1}^T\E\left[\delta_t^2\right]\leq \sum_{t=1}^{T}\E\left[\frac{3\left(\delta_t^2 -\delta_{t+1}^2\right)}{\beta}\right] +\sum_{t=1}^{T} 5\beta \sigma_m^2 +  \sum_{t=1}^T\E\left[\frac{1}{8}\|\nabla F(\w_t)\|^2\right].
	\end{equation}
	
	By Assumption~\ref{ass:1} (i) and by the property of $L$-smooth function, we know that
	\begin{equation*}
	\begin{aligned}
	F(\w_{t+1})&\leq F(\w_t)+\nabla^\top F(\w_t)(\w_{t+1}-\w_t)+\frac{L}{2}\left\|\w_{t+1}-\w_t\right\|^2\\
	&{=}F(\w_t)-\eta_t\nabla^\top F(\w_t)\z_t+\frac{\eta_t^2L}{2}\left\|\z_t\right\|^2\\
	&\leq F(\w_t)-\eta_t\nabla^\top F(\w_t)(\z_t-\nabla F(\w_t)+\nabla F(\w_t))+\eta_t^2 L
	\left(\left\|\z_t-\nabla F(\w_t)\right\|^2+\left\|\nabla F(\w_t)\right\|^2\right)\\
	&=F(\w_t)-(\eta_t-\eta_t^2 L)\|\nabla F(\w_t)\|^2-\eta_t\nabla^\top F(\w_t)(\z_t-\nabla F(\w_t))+\eta_t^2 L\left\|\z_t-F(\w_t)\right\|^2\\
	&\leq F(\w_t)-\left(\frac{\eta_t}{2}-\eta_t^2 L\right)\|\nabla F(\w_t)\|^2+\left(\frac{\eta_t}{2}+\eta_t^2 L\right) \left\|\z_t-F(\w_t)\right\|^2.
	\end{aligned}
	\end{equation*}
	Noting that $\eta_t=\frac{\alpha\beta^a}{\max\left(\|\z_t\|^{1/2},\epsilon_0\right)}$, $\alpha\leq 1/4L$ and $\epsilon_0=\beta^a$, we know that $\eta_t L\leq 1/4$. Hence we have 
	\[
	\left\|\nabla F(\w_t)\right\|^2\leq \frac{4(F(\w_t)-F(\w_{t+1}))}{\eta_t}+3\|\z_t-\nabla F(\w_t)\|^2.
	\]
	Taking summation over $t=1,\ldots, T$ and taking expectation yield
	\begin{equation}
	\label{mainproof:ineq4}
	\begin{aligned}
	\sum_{t=1}^{T}\E\left\|\nabla F(\w_t)\right\|^2&\leq \E\left[\sum_{t=1}^{T}\frac{4(F(\w_t)-F(\w_{t+1}))}{\eta_t}\right]+3\sum_{t=1}^{T}\E\|\z_t-\nabla F(\w_t)\|^2\\
	&\leq \E\left[\sum_{t=1}^{T}\frac{4(F(\w_t)-F(\w_{t+1}))}{\eta_t}\right]+3\sum_{t=1}^{T}\E\left[\delta_t^2\right].
	\end{aligned}
	\end{equation}
	Combining~(\ref{mainproof:ineq3}) and~(\ref{mainproof:ineq4}) yields
	\[
	\begin{aligned}
	\sum_{t=1}^{T}\E\left\|\nabla F(\w_t)\right\|^2 
	&\leq \E\left[\sum_{t=1}^T\frac{4(F(\w_t)-F(\w_{t+1}))}{\eta_t}\right]+\sum_{t=1}^{T}\E\left[\frac{9\left(\delta_t^2 -\delta_{t+1}^2\right)}{\beta}\right] \\
	&\quad +\sum_{t=1}^{T} 15\beta \sigma_m^2 +  \sum_{t=1}^T\E\left[\frac{3}{8}\|\nabla F(\w_t)\|^2\right].
	\end{aligned}
	\]
	By some simple algebra, we have 
	\[
	\begin{aligned}
	\sum_{t=1}^{T}\E\left\|\nabla F(\w_t)\right\|^2\leq \E\left[\sum_{t=1}^T\frac{8(F(\w_t)-F(\w_{t+1}))}{\eta_t}\right]+\sum_{t=1}^{T}\E\left[\frac{18\left(\delta_t^2 -\delta_{t+1}^2\right)}{\beta}\right] +\sum_{t=1}^{T} 30\beta \sigma_m^2.
	\end{aligned}
	\]
	Then we have
	\begin{equation}
	\label{ineq:11111}
	\frac{1}{T}\sum_{t=1}^{T}\E\left\|\nabla F(\w_t)\right\|^2\leq \E\left[\sum_{t=1}^T\frac{8\max\left(\|\z_t\|^{1/2},\epsilon_0\right)(F(\w_t)-F(\w_{t+1}))}{\alpha\beta^a T}\right]+\frac{18\sigma_0^2}{\beta T} + 30\beta \sigma_m^2.  
	\end{equation}
	Noting that $|F(\w_t)-F(\w_{t+1})|\leq G\eta_t\|\z_t\|$, we have
	\begin{equation}
	\label{ineq:11121}
	\frac{1}{T}\sum_{t=1}^{T}\E\left\|\nabla F(\w_t)\right\|^2\leq \frac{8G\E\left[\sum_{t=1}^{T}\|\z_t\|\right]}{T}+\frac{\Delta}{\alpha T}+\frac{18\sigma_0^2}{\beta T} + 30\beta \sigma_m^2. \qedhere
	\end{equation}
	
\end{proof}

\section{Proof of Theorem~\ref{thm:2}}
Before introducing the proof, we first introduce several lemmas which are useful for our analysis.
\begin{lemma}
	\label{lem1:1}
	Adam$^{+}$ with $\eta_t=\frac{\alpha\beta}{\max\left(\|\z_t\|^{1/2},\epsilon_0\right)}$ and $\epsilon_0=0$ satisfies
	\begin{equation*}
	F(\w_{t+1})-F(\w_t)\leq \alpha\beta\left(-\frac{\|\nabla F(\w_t)\|^{3/2}}{6}+9\|\z_t-\nabla F(\w_t)\|^{3/2} \right)+\frac{64\alpha^4\beta^4L^3}{3}.
	\end{equation*}
\end{lemma}
\begin{proof}
	By the $L$-smoothness and the update of the algorithm, we have
	\begin{equation}
	\label{newproof:ineq11}
	\begin{aligned}
	F(\w_{t+1})- F(\w_t) &\leq \nabla^\top F(\w_t)(\w_{t+1}-\w_t)+\frac{L\|\w_{t+1}-\w_t\|^2}{2} \\
	&\leq -\alpha\beta \cdot\frac{\langle \nabla F(\w_t), \z_t\rangle}{\max\left(\|\z_t\|^{1/2},\epsilon_0\right)}+\frac{\alpha^2\beta^2 L\|\z_t\|^2}{\left(\max\left(\|\z_t\|^{1/2},\epsilon_0\right)\right)^2}.
	\end{aligned}
	\end{equation}
	Define $\Delta_t=\z_t-\nabla F(\w_t)$. If $\|\nabla F(\w_t)\|\geq 2\|\Delta_t\|$, we have
	\begin{equation}
	\label{newproof:ineq1}
	\begin{aligned}
	-\frac{\langle \z_t,\nabla F(\w_t)\rangle}{\max\left(\|\z_t\|^{1/2},\epsilon_0\right)} 
	&=-\frac{\|\nabla F(\w_t)\|^2+\langle \Delta_t, \nabla F(\w_t)\rangle}{\max\left(\|\nabla F(\w_t)+\Delta_t\|^{1/2},\epsilon_0\right)}\\
	&\leq -\frac{\|\nabla F(\w_t)\|^2}{2\|\nabla F(\w_t)+\Delta_t\|^{1/2}} \leq -\frac{\|\nabla F(\w_t)\|^{3/2}}{3} \\
	&\leq -\frac{\|\nabla F(\w_t)\|^{3/2}}{3}+8\|\Delta_t\|^{3/2}.
	\end{aligned}
	\end{equation} 
	If $\|\nabla F(\w_t)\|\leq 2\|\Delta_t\|$, we have
	\begin{equation}
	\label{newproof:ineq2}
	\begin{aligned}
	-\frac{\langle \z_t,\nabla F(\w_t)\rangle}{\max\left(\|\z_t\|^{1/2},\epsilon_0\right)}
	&=-\frac{\|\nabla F(\w_t)\|^2+\langle \Delta_t, \nabla F(\w_t)\rangle}{\max\left(\|\nabla F(\w_t)+\Delta_t\|^{1/2},\epsilon_0\right)} \\
	&\leq \frac{6\|\Delta_t\|^2}{\|\Delta_t\|^{1/2}} 
	=6\|\Delta_t\|^{3/2} \leq -\frac{\|\nabla F(\w_t)\|^{3/2}}{3}+8\|\Delta_t\|^{3/2}.
	\end{aligned}
	\end{equation}
	By~(\ref{newproof:ineq1}) and~(\ref{newproof:ineq2}), we have
	\begin{equation}
	\label{newproof:ineq3}
	-\frac{\langle \z_t,\nabla F(\w_t)\rangle}{\max\left(\|\z_t\|^{1/2},\epsilon_0\right)
	}\leq -\frac{\|\nabla F(\w_t)\|^{3/2}}{3}+8\|\Delta_t\|^{3/2}.
	\end{equation}
	By~(\ref{newproof:ineq11}) and~(\ref{newproof:ineq3}), we have
	\begin{equation*}
	\begin{aligned}
	&\quad F(\w_{t+1})-F(\w_t)\leq \alpha\beta\left(\frac{\|\nabla F(\w_t)\|^{3/2}}{3}+8\|\Delta_t\|^{3/2} \right)+\alpha^2\beta^2 L\|\z_t\|\\
	&= \alpha\beta\left(-\frac{\|\nabla F(\w_t)\|^{3/2}}{3}+8\|\Delta_t\|^{3/2} \right)+\alpha^2\beta^2 L \min_{x>0}\left(\frac{2\|\z_t\|^{3/2}}{3x}+\frac{x^2}{3}\right)\\
	&\leq \alpha\beta\left(-\frac{\|\nabla F(\w_t)\|^{3/2}}{3}+8\|\Delta_t\|^{3/2} \right)+\alpha^2\beta^2 L \left(\frac{2\|\z_t\|^{3/2}}{3 (8\alpha\beta L)}+\frac{64\alpha^2\beta^2L^2}{3}\right)\\
	&\leq \alpha\beta\left(-\frac{\|\nabla F(\w_t)\|^{3/2}}{6}+9\|\Delta_t\|^{3/2} \right)+\frac{64\alpha^4\beta^4L^3}{3},
	\end{aligned}
	\end{equation*}
	where the last inequality holds because $\|\z_t\|^{3/2}\leq  2\|\nabla F(\w_t)\|^{3/2}+2\|\Delta_t\|^{3/2}$.
\end{proof}
\begin{lemma}
	\label{lem1:2/3}
	For Adam$^{+}$ with $\eta_t=\frac{\alpha\beta}{\max\left(\|\z_t\|^{1/2},\epsilon_0\right)}$, there exist random variables $\delta_t$ such that
	\begin{equation*}
	\E\left[\delta_{t+1}^{3/2}\right] \leq \left(1- \frac{\beta}{2}\right)\E\left[\delta_t^{3/2}\right] +2\beta^{3/2} \sigma^{3/2} + \E\left[\frac{320L_H^{3/2}\|\w_{t+1} - \w_t\|^3}{\beta^2}\right].
	\end{equation*}
\end{lemma}
\begin{proof}
	The proof shares the similar spirit of Lemma 12 in~\citep{wang2017stochastic}, but we adapt the proof for our purpose. Define
	\begin{equation}
	\begin{aligned}
	\zeta_k^{(t)}=\begin{cases}
	\beta (1-\beta)^{t-k} &  \text{if  } t\geq k> 0\\
	(1-\beta)^{t-k}					& \text{if } t\geq k=0
	\end{cases}\\
	\end{aligned}
	\end{equation}
	By the definition of $\zeta_t^{(k)}$ and the update of Algorithm~\ref{Alg:1}, we have 
	$$\zeta_k^{(t+1)}=(1-\beta)\zeta_k^{(t)}, \quad \sum_{k=0}^{t}\zeta_k^{(t)}=1,\quad \w_t=\sum_{k=0}^{t}\zeta_k^{(t)}\wh_{t+1},\quad \z_{t+1}=\sum_{k=0}^{t}\zeta_k^{(t)}\nabla f(\wh_{t+1};\xi_{t+1}).$$
	Define $m_{t+1}=\sum_{k=0}^{t}\zeta_k^{(t)}\|\w_{t+1}-\wh_{k+1}\|^2$, $n_{t+1}=\sum_{k=0}^{t}\zeta_k^{(t)}\left[\nabla f(\wh_{k+1};\xi_{k+1})-F(\wh_{k+1})\right]$, where $\nabla f(\wh_{k+1};\xi_{k+1})$ is an unbiased stochastic first-order oracle for $F(\wh_{k+1})$ with bounded variance $\sigma^2$.
	Note that $\nabla F$ is a $L_H$-smooth mapping (according to Assumption~\ref{ass:1}), then by Lemma 10 of~\citep{wang2017stochastic}, we have
	\begin{equation*}
	\|\z_t-\nabla F(\w_t)\|^{3/2}\leq (L_Hm_t+\|n_t\|)^{3/2}\leq 2L_H^{3/2}m_t^{3/2}+2\|n_t\|^{3/2}.
	\end{equation*}
	Define $q_{t+1}=\sum_{k=0}^{t}\zeta_k^{(t)}\left\|\w_{t+1}-\wh_{k+1}\right\|$. According to Lemma 11 (a) and (b) of~\citep{wang2017stochastic}, we have
	\begin{equation*}
	m_{t+1}+4q_{t+1}^2\leq \left(1-\frac{\beta}{2}\right)\left(m_t+4q_t^2\right)+\frac{18}{\beta}\|\w_{t+1}-\w_t\|^2.
	\end{equation*}
	Taking the power $3/2$ on both sides of the inequality and using the fact that $(a+b)^{3/2}\leq \sqrt{1+\frac{\beta}{2}}a^{3/2}+\sqrt{1+\frac{2}{\beta}}b^{3/2}$ for $\beta>0$, we have
	\begin{equation}
	\label{inequality:11}
	\begin{aligned}
	&\left(m_{t+1}+4q_{t+1}^2\right)^{3/2}\\
	&\quad \leq \left(1+\frac{\beta}{2}\right)^{1/2}\left(1-\frac{\beta}{2}\right)^{3/2}\left(m_t+4q_t^2\right)^{3/2}+\left(1+\frac{2}{\beta}\right)^{1/2}\frac{80}{\beta^{3/2}}\|\w_{t+1}-\w_t\|^3\\
	&\quad \leq \left(1-\frac{\beta}{2}\right)(m_t+4q_t^2)^{3/2}+\frac{160}{\beta^2}\left\|\w_{t+1}-\w_t\right\|^3,
	\end{aligned}
	\end{equation}
	where the last inequality holds since $1/\beta\geq 1$.
	
	By the definition of $n_t$, we have
	$n_{t+1}=(1-\beta)n_t+\beta(\nabla f(\wh_{t+1})-F(\wh_{t+1}))$.
	Denote $\mathcal{F}_{t+1}$ by the $\sigma$-algebra generated by $\xi_1,\ldots,\xi_{t+1}$. Noting that 
	\begin{equation}
	\label{inequality:12}
	\E\left[\|n_{t+1}\|^{3/2}\vert \mathcal{F}_{t+1}\right]\leq \left(\E\left[\|n_{t+1}\|^2 \vert \mathcal{F}_{t+1}\right]\right)^{3/4}\leq (1-\beta/2)^{3/2}\|n_t\|^{3/2}+\beta^{3/2}\sigma^{3/2},
	\end{equation}
	where the last inequality holds by invoking Lemma 11(c) of~\citep{wang2017stochastic}.
	Define $\delta_t^{3/2}=2L_H^{3/2}(m_t+4q_t^2)^{3/2}+2\|n_t\|^{3/2}$, then we have $\left\|\z_t-\nabla F(\w_t)\right\|^{3/2}\leq \delta_t^{3/2}$ for all $t$. 
	According to~(\ref{inequality:11}) and~(\ref{inequality:12}), we have 
	\[
	\E\left[\delta_{t+1}^{3/2}\vert \mathcal{F}_{t+1}\right]\leq \left(1-\frac{\beta}{2}\right)\|\delta_t\|^{3/2}+2\beta^{3/2}\sigma^{3/2}+\frac{320L_H^{3/2}\|\w_{t+1}-\w_t\|^3}{\beta^2}.
	\]
	
	Taking expectation on both sides yields
	\[
	\E\left[\delta_{t+1}^{3/2}\right] \leq \left(1- \frac{\beta}{2}\right)\E\left[\delta_t^{3/2}\right] +2\beta^{3/2} \sigma^{3/2} + \E\left[\frac{320L_H^{3/2}\|\w_{t+1} - \w_t\|^3}{\beta^2}\right]. \qedhere
	\]
\end{proof}

\begin{lemma}
	\label{lem:4}
	Adam$^{+}$ with learning rate  $\eta_t=\frac{\alpha\beta}{\max\left(\|\z_t\|^{1/2},\epsilon_0\right)}$ and $640\alpha^3 L_H^{3/2}\leq 1/120$ satisfies
	\begin{equation*}
	\frac{1}{T}\sum_{t=1}^{T} \E\left[\|\nabla F(\w_t)\|^{3/2}\right] \leq \frac{101\Delta}{\alpha\beta T}+\frac{2727\E\left[\delta_1^{3/2}\right]}{\beta T}+4545\beta^{1/2}\sigma^{3/2}+\frac{3\beta^3L^{3/2}}{100}.
	\end{equation*}
	To ensure that $\frac{1}{T}\sum_{t=1}^{T} \E\left[\|\nabla F(\w_t)\|^{3/2}\right]\leq \epsilon^{3/2}$, we can choose $\beta=\epsilon^{3}$, $T=O(\epsilon^{-9/2})$. 
\end{lemma}
\begin{proof}
	By Lemma~\ref{lem1:2/3} and noting that $\eta_t=\frac{\alpha\beta}{\max\left(\|\z_t\|^{1/2},\epsilon_0\right)}$, we have
	\begin{equation}
	\label{inequality:13}
	\begin{aligned}
	&	\E\left[\delta_{t+1}^{3/2}\right] \leq \left(1- \frac{\beta}{2}\right)\E\left[\delta_t^{3/2}\right] +2\beta^{3/2} \sigma^{3/2} + \E\left[\frac{320L_H^{3/2}\alpha^3\beta^3\|\z_t\|^{3/2}}{\beta^2}\right]\\
	&\leq \left(1- \frac{\beta}{2}\right)\E\left[\delta_t^{3/2}\right] +2\beta^{3/2} \sigma^{3/2} + \E\left[640L_H^{3/2}\alpha^3\beta\left(\|\nabla F(\w_t)\|^{3/2}+\|\delta_t\|^{3/2}\right)\right].
	\end{aligned}
	\end{equation}
	Note that $640\alpha^3 L_H^{3/2}\leq 1/120$. Plugging it into~(\ref{inequality:13}), we have
	\begin{equation}
	\label{inequality:14}
	\frac{59\beta}{120}\E\left[\delta_t^{3/2}\right]\leq \E\left[\delta_t^{3/2}-\delta_{t+1}^{3/2}\right]+2\beta^{3/2}\sigma^{3/2}+\E\left[\frac{\beta}{120}\|\nabla F(\w_t)\|^{3/2}\right].
	\end{equation}
	Summing over $t=1,\ldots, T$ on both sides of~(\ref{inequality:14}) and with some simple algebra, we have
	\begin{equation*}
	\sum_{t=1}^{T}\E\left[\delta_t^{3/2}\right]\leq \sum_{t=1}^{T}\E\left[\frac{3(\delta_t^{3/2}-\delta_{t+1}^{3/2})}{\beta}\right]+\sum_{t=1}^{T}5\beta^{1/2}\sigma^{3/2}+\sum_{t=1}^{T}\E\left[\frac{1}{59}\|\nabla F(\w_t)\|^{2/3}\right].
	\end{equation*}
	By Lemma~\ref{lem1:1}, taking expectation on both sides, we have
	\begin{equation}
	\label{inequality:15}
	\begin{aligned}
	\E\left[F(\w_{t+1})-F(\w_t)\right]\leq \alpha\beta\left(-\frac{\E\left[\|\nabla F(\w_t)\|^{3/2}\right]}{6}+9\E\left[\delta_t^{3/2}\right] \right)+\frac{64\alpha^4\beta^4L^3}{3}.
	\end{aligned}
	\end{equation}
	Summing~(\ref{inequality:15}) over $t=1,\ldots,T$ yields
	\begin{equation*}
	\frac{5}{504}\alpha\beta \sum_{t=1}^{T} \E\left[\|\nabla F(\w_t)\|^{3/2}\right]\leq F(\w_1)-F_* + \alpha\beta\left(\frac{27\E\left[\delta_1^{3/2}\right]}{\beta}+\sum_{t=1}^{T}45\beta^{1/2}\sigma^{3/2}\right)+\frac{64\alpha^4\beta^4L^3T}{3}.
	\end{equation*}
	Hence, we have
	\begin{align*}
	\frac{1}{T}\sum_{t=1}^{T} \E\left[\|\nabla F(\w_t)\|^{3/2}\right] &\leq \frac{101\Delta}{\alpha\beta T}+\frac{2727\E\left[\delta_1^{3/2}\right]}{\beta}+4545\beta^{1/2}\sigma^{3/2}+2155\alpha^3\beta^3L^3\\
	&\leq \frac{101\Delta}{\alpha\beta T}+\frac{2727\E\left[\delta_1^{3/2}\right]}{\beta T}+4545\beta^{1/2}\sigma^{3/2}+\frac{3\beta^3L^{3/2}}{100}. \qedhere
	\end{align*}
\end{proof}
\begin{lemma}
	\label{lemma:5}
	Under the same setting of Lemma~\ref{lem:4}, we know that to ensure that
	$\frac{1}{T}\sum_{t=1}^{T}\E\left[\delta_t^{3/2}\right]\leq \epsilon^{3/2}$, we need $T=O(\epsilon^{-9/2})$ iterations.
\end{lemma}
\begin{proof}
	From (\ref{inequality:14}) and Lemma~\ref{lem:4}, we have
	\begin{equation*}
	\sum_{t=1}^T \frac{59\beta}{120}\E\left[\delta_t^{3/2}\right]\leq \E\left[\delta_1^{3/2}\right]+2\beta^{3/2}\sigma^{3/2}T+\sum_{t=1}^{T}\E\left[\frac{\beta}{120}\|\nabla F(\w_t)\|^{3/2}\right].
	\end{equation*}
	Noting that $\beta=T^{-b}$ with $0< b< 1$, then we know that there exists a universal constant $C>0$  such that
	\begin{equation}
	\label{ineq:1}
	\frac{1}{T}\sum_{t=1}^T \frac{59}{120}\E\left[\delta_t^{3/2}\right]\leq  \frac{\E\left[\delta_1^{3/2}\right]}{T^{1-b}}+\frac{2\sigma^{3/2}}{T^{b/2}}+ \frac{1}{T}\sum_{t=1}^{T}\E\left[\frac{1}{120}\|\nabla F(\w_t)\|^{3/2}\right].
	\end{equation}
	Take $b=\frac{2}{3}$. From Lemma~\ref{lem:4}, we know that it takes $T=O(\epsilon^{-9/2})$ iterations to ensure that $\frac{1}{T}\sum_{t=1}^{T}\E\left[\|\nabla F(\w_t)\|^{3/2}\right]\leq \epsilon^{3/2}$. In addition, From~(\ref{ineq:1}), we know that it takes $T=O(\epsilon^{-9/2})$ iterations to ensure that $\frac{1}{T}\sum_{t=1}^{T}\E\left[\delta_t^{3/2}\right]\leq \epsilon^{3/2}$.
\end{proof}

We can easily prove Theorem~\ref{thm:2} by incorporating the results in Lemma~\ref{lem:4} and Lemma~\ref{lemma:5}. It is also evident to see that if $\beta=1/T^s$ with $0<s<1$, then it takes $T=O\left(\text{poly}(1/\epsilon)\right)$ number of iterations to ensure that 
$\frac{1}{T}\sum_{t=1}^{T}\E\left[\delta_t^{3/2}\right]\leq \epsilon^{3/2}$ and $\frac{1}{T}\sum_{t=1}^{T}\E\left[\|\nabla F(\w_t)\|^{3/2}\right]\leq \epsilon^{3/2}$ hold simultaneously.



\section{Proof of Theorem~\ref{thm3}}
\begin{proof}
	Define $\gamma_t=\min\left(\frac{\beta^a}{\|\z_t\|^{2/3}}, \frac{\beta^a}{\epsilon_0}\right)$ with $\epsilon_0=2\beta^a$. Then we know that $\eta_t=\alpha\gamma_t$ and $\gamma_t\leq \frac{1}{2}$. Note that $\alpha\leq\frac{1}{L}$, so we have $\eta_t\leq \frac{1}{2L}$.
	By the $L$-smoothness of $F$, we have
	\begin{equation}
	\label{fastrate:eq1}
	\begin{aligned}
	&\quad F(\w_{t+1})\leq F(\w_t)+\nabla^\top F(\w_t)(\w_{t+1}-\w_t)+\frac{L}{2}\left\|\w_{t+1}-\w_t\right\|^2\\
	&\leq F(\w_t)-\eta_t\nabla^\top F(\w_t)\z_t+\left(\frac{\eta_t^2L}{2}+\frac{\gamma_t}{2L}\right)\|\z_t\|^2-\frac{1}{2L}\gamma_t\|\z_t\|^2\\
	&= F(\w_t)-\eta_t\nabla^\top F(\w_t)\left(\z_t-\nabla F(\w_t)+\nabla F(\w_t)\right)+\left(\frac{\eta_t^2L}{2}+\frac{\gamma_t}{2L}\right)\|\z_t\|^2-\frac{1}{2L}\gamma_t\|\z_t\|^2\\
	&\stackrel{(a)}{\leq} F(\w_t)-\eta_t\nabla^\top F(\w_t)\left(\z_t-\nabla F(\w_t)+\nabla F(\w_t)\right)+\left(\eta_t^2L+\frac{\gamma_t}{L}\right)\left(\|\z_t-\nabla F(\w_t)\|^2+\left\|\nabla F(\w_t)\right\|^2\right)-\frac{1}{2L}\gamma_t\|\z_t\|^2\\
	&\stackrel{(b)}{\leq} F(\w_t)-\frac{\eta_t}{2}\left\|\nabla F(\w_t)\right\|^2 +\frac{\eta_t}{2}\|\z_t-\nabla F(\w_t)\|^2+\left(\eta_t^2L+\frac{\gamma_t}{L}\right)\left(\|\z_t-\nabla F(\w_t)\|^2+\left\|\nabla F(\w_t)\right\|^2\right)-\frac{1}{2L}\gamma_t\|\z_t\|^2\\
	&=F(\w_t)-\left(\frac{\eta_t}{2}-\eta_t^2 L-\frac{\gamma_t}{L}\right)\left\|\nabla F(\w_t)\right\|^2 +\left(\eta_t^2L+\frac{\gamma_t}{L}+\frac{\eta_t}{2}\right)\|\z_t-\nabla F(\w_t)\|^2-\frac{1}{2L}\gamma_t\|\z_t\|^2\\
	&\stackrel{(c)}{\leq} F(\w_t)-\frac{1}{2L}\gamma_t\|\z_t\|^2+\frac{1}{L}\left\|\z_t-\nabla F(\w_t)\right\|^2,
	\end{aligned}	
	\end{equation}
	where (a) holds since $\|\z_t\|^2\leq 2\|\z_t-\nabla F(\w_t)\|^2+2\|\nabla F(\w_t)\|^2$, (b) holds since $-\nabla^\top F(\w_t)\z_t\leq \frac{1}{2}\left(\|\nabla F(\w_t)\|^2+\|\z_t-\nabla F(\w_t)\|^2\right)$, (c) holds due to $\frac{\eta_t}{2}-\eta_t^2 L-\frac{\gamma_t}{L}\geq 0$ (since $\eta_t\leq\frac{1}{2L}$, we have $\frac{\eta_t}{2}-\eta_t^2 L\geq \frac{1}{2L}$ and note that $\frac{\gamma_t}{L}\leq \frac{1}{2L}$). 
	
	By the definition of $\gamma_t$, we have
	\begin{equation}
	\label{fastrate:eq2}
	\begin{aligned}
	\gamma_t\|\z_t\|^2&\geq \beta^{2a}\|\z_t\|^{2/3}\min\left(\frac{\|\z_t\|^{2/3}}{\beta^a}, \frac{\|\z_t\|^{4/3}}{\beta^a\epsilon_0}\right)=\beta^{2a}\|\z_t\|^{2/3}\min\left(\frac{\|\z_t\|^{2/3}}{\beta^a}, \frac{\|\z_t\|^{4/3}}{2\beta^{2a}}\right)\\
	&\stackrel{(a)}{\geq} \beta^{2a}\|\z_t\|^{2/3}\left(\frac{\|\z_t\|^{2/3}}{\beta^a}-\frac{1}{2}\right)=\beta^a\|\z_t\|^{4/3}-\frac{\beta^{2a}\|\z_t\|^{2/3}}{2},
	\end{aligned}	
	\end{equation} 
	where (a) holds since $x\geq x-\frac{1}{2}$, $\frac{x^2}{2}\geq x-\frac{1}{2}$ hold for any $x$ and let $x=\frac{\|\z_t\|^{2/3}}{\beta^a}$.
	
	Combining~(\ref{fastrate:eq1}) and~(\ref{fastrate:eq2}), we have
	\begin{equation*}
	\begin{aligned}
	\beta^a\|\z_t\|^{4/3}&\leq \gamma_t\|\z_t\|^2+\frac{\beta^{2a}\|\z_t\|^{2/3}}{2}
	\leq 2L\left(F(\w_t)-F(\w_{t+1})\right)+\frac{\beta^{2a}\|\z_t\|^{2/3}}{2}+2\|\z_t-\nabla F(\w_t)\|^2\\
	&=2L\left(F(\w_t)-F(\w_{t+1})\right)+\beta^{a}\|\z_t\|^{4/3}\cdot\frac{\beta^a}{2\|\z_t\|^{2/3}}+2\|\z_t-\nabla F(\w_t)\|^2.
	\end{aligned}
	\end{equation*}
	If $\frac{\beta^{a}}{2\|\z_t\|^{2/3}}\leq \frac{1}{2}$, we have $	\beta^a\|\z_t\|^{4/3}\leq 4L\left(F(\w_t)-F(\w_{t+1})\right)+4\|\z_t-\nabla F(\w_t)\|^2.$
	If $\frac{\beta^{a}}{2\|\z_t\|^{2/3}}> \frac{1}{2}$, then $\beta^a>\|\z_t\|^{2/3}$, and hence we have $\beta^a\|\z_t\|^{4/3}\leq \beta^{3a}$. As a result, we have
	\begin{equation}
	\label{fastrate:eq3}
	\begin{aligned}
	\beta^a\|\z_t\|^{4/3}\leq  4L\left(F(\w_t)-F(\w_{t+1})\right)+4\|\z_t-\nabla F(\w_t)\|^2+\beta^{3a}.
	\end{aligned}
	\end{equation}
	Taking summation on both sides of~(\ref{fastrate:eq3}) over $t=1,\ldots,T$ yields
	\begin{equation}
	\label{fastrate:eq4}
	\sum_{t=1}^{T}\|\z_t\|^{4/3}\leq  4L\sum_{t=1}^{T}\frac{F(\w_t)-F(\w_{t+1})}{\beta^a}+\sum_{t=1}^{T}\frac{4}{\beta^a}\|\z_t-\nabla F(\w_t)\|^2+\beta^{2a}T.
	\end{equation}
	Define $\Delta_t=\z_t-\nabla F(\w_t)$, then we have
	\begin{equation}
	\label{fastrate:eq5}
	\|\nabla F(\w_t)\|^{4/3}\leq 2\|\z_t\|^{4/3}+2\|\Delta_t\|^{4/3}.
	\end{equation}
	Hence,
	\begin{equation}
	\label{fastrate:eq6}
	\begin{aligned}
	&\quad\sum_{t=1}^{T}\|\z_t\|^{4/3}+\|\nabla F(\w_t)\|^{4/3} \\
	&\stackrel{(a)}{\leq }2\sum_{t=1}^{T}\|\Delta_t\|^{4/3}+12L\sum_{t=1}^{T}\frac{F(\w_t)-F(\w_{t+1})}{\beta^a}+\sum_{t=1}^{T}\frac{12}{\beta^a}\|\z_t-\nabla F(\w_t)\|^2+3\beta^{2a}T\\
	&\stackrel{(b)}{\leq} \sum_{t=1}^{T}\frac{4}{3}\left(\frac{\|\Delta_t\|^2}{\beta^a}+\frac{\beta^{2a}}{2}\right)+12L\sum_{t=1}^{T}\frac{F(\w_t)-F(\w_{t+1})}{\beta^a}+\sum_{t=1}^{T}\frac{12}{\beta^a}\|\z_t-\nabla F(\w_t)\|^2+3\beta^{2a}T\\
	&\leq 12L\sum_{t=1}^{T}\frac{F(\w_t)-F(\w_{t+1})}{\beta^a}+\sum_{t=1}^{T}\frac{14}{\beta^a}\|\z_t-\nabla F(\w_t)\|^2+4\beta^{2a}T,
	\end{aligned}
	\end{equation}
	where (a) holds due to~(\ref{fastrate:eq4}) and~(\ref{fastrate:eq5}), (b) holds because $\min_{x>0}\frac{c^2}{x}+\frac{x^2}{2}=\frac{3c^{4/3}}{2}$.
	
	By Lemma~\ref{lemma:1}, we know that
	\begin{equation*}
	\begin{aligned}
	\E\left[\delta_{t+1}^2\right] &\leq \left(1- \frac{\beta}{2}\right)\E\left[\delta_t^2\right] +2\beta^2 \sigma^2 + \E\left[\frac{CL_H^2\eta_t^4\|\z_t\|^4}{\beta^3}\right]\\
	&\stackrel{(a)}{\leq}\left(1- \frac{\beta}{2}\right)\E\left[\delta_t^2\right] +2\beta^2 \sigma^2 + \E\left[\frac{CL_H^2\alpha^4\beta^{4a}\|\z_t\|^4}{\max(\|\z_t\|^{8/3},\epsilon_0^4)\beta^3}\right]\\
	&\leq \left(1- \frac{\beta}{2}\right)\E\left[\delta_t^2\right] +2\beta^2 \sigma^2 + \E\left[CL_H^2\alpha^4\beta^{4a-3}\|\z_t\|^{4/3}\right].
	\end{aligned}
	\end{equation*}
	Note that $CL_H^2\alpha^4\leq 1/14$, we have
	\begin{equation}
	\label{fastrate:eq7}
	\frac{\beta}{2}\E\left[\delta_t^2\right]\leq \E\left[\delta_t^2-\delta_{t+1}^2\right] +2\beta^2 \sigma^2 + \E\left[\frac{\beta^{4a-3}\|\z_t\|^{4/3}}{14}\right].
	\end{equation}
	Taking summation on both sides of~(\ref{fastrate:eq7}) over $t=1,\ldots,T$, we have
	\begin{equation}
	\label{fastrate:eq8}
	\begin{aligned}
	\sum_{t=1}^{T}\E\left[\delta_t^2\right]
	&\leq \frac{\E\left[\delta_1^2\right]}{\beta}+2\beta\sigma^2T+\sum_{t=1}^{T}\E\left[\frac{\beta^{4a-4}\|\z_t\|^{4/3}}{14}\right]\\
	&=\frac{\E\left[\delta_1^2\right]}{\beta}+2\beta\sigma^2T+\sum_{t=1}^{T}\E\left[\frac{\beta^{a}\|\z_t\|^{4/3}}{14}\right],
	\end{aligned}
	\end{equation}
	where the last equality holds since $a=4/3$.
	
	Taking expectation on both sides of~(\ref{fastrate:eq6}) and combining~(\ref{fastrate:eq8}), we have
	\begin{equation*}
	\begin{aligned}
	\sum_{t=1}^{T}\E\left[\|\z_t\|^{4/3}+\|\nabla F(\w_t)\|^{4/3}\right]\leq \frac{12L\Delta}{\beta^a}+\frac{14\E\left[\delta_1^2\right]}{\beta^{1+a}}+\frac{28\beta\sigma^2 T}{\beta^a}+\sum_{t=1}^{T}\E\left[\|\z_t\|^{4/3}\right]+4\beta^{2a}T.
	\end{aligned}
	\end{equation*}
	As a result, we have
	\begin{equation*}
	\frac{1}{T}\sum_{t=1}^{T}\E\left[\|\nabla F(\w_t)\|^{4/3}\right]\leq \frac{12L\Delta}{\beta^a T}+\frac{14\E\left[\delta_1^2\right]}{\beta^{1+a} T}+\frac{28\beta\sigma^2 }{\beta^a}+4\beta^{2a}.
	\end{equation*}
	Suppose initial batch size is $T_0$, the intermediate batch size is $m$, and $a=4/3$, then we have
	\begin{equation}
	\frac{1}{T}\sum_{t=1}^{T}\E\left[\|\nabla F(\w_t)\|^{4/3}\right]\leq \frac{12L\Delta}{\beta^{4/3} T}+\frac{14\sigma^2}{\beta^{7/3}T_0T}+\frac{28\sigma^2 }{\beta^{1/3} m}+4\beta^{8/3}.
	\end{equation}
	We can choose $\beta=O(\epsilon^{1/2})$, $T=O(\epsilon^{-2})$,  the initial batch size $T_0=1/\beta=O(\epsilon^{-1/2})$, the intermediate batch size as $m=1/\beta^3=O(\epsilon^{-3/2})$, which ends up with the total complexity $O(\epsilon^{-3.5})$.
\end{proof}

\section{A New Variant of Adam$^+$}
\label{sec:newvariant}
\begin{thm}
	Assume that $\|\nabla f(\w;\xi)\|\leq G$ almost surely for every $\w\in\R^d$. Choose $\eta_t=\frac{\alpha\beta^a}{\max\left(\|\z_t\|^{1/2},\epsilon_0\right)}$ with $a=4/3$, and we have
	\begin{equation*}
	\frac{1}{T}\sum_{t=1}^{T}\E\left[\|\nabla F(\w_t)\|^{3/2}\right]\leq \frac{12L\Delta}{\beta^a T}+\frac{14\E\left[\delta_1^2\right]}{\beta^{1+a} T}+\frac{28\beta\sigma^2 }{\beta^a}+4\beta^{3a}.
	\end{equation*}
	Denote the initial batch size and the intermediate batch size are $T_0$ and $m$ respectively, then we have
	\begin{equation*}
	\frac{1}{T}\sum_{t=1}^{T}\E\left[\|\nabla F(\w_t)\|^{3/2}\right]\leq \frac{12L\Delta}{\beta^a T}+\frac{14\sigma^2}{T_0 T\beta^{1+a}}+\frac{28\sigma^2 }{\beta^{a-1}m}+4\beta^{3a}.
	\end{equation*}
	
	To ensure that $\frac{1}{T}\sum_{t=1}^{T}\E\left[\|\nabla F(\w_t)\|^{3/2}\right]\leq \epsilon^{3/2}$, we choose $\beta=\epsilon^{3/8}$, $T=O(1/\epsilon^2)$, the initial batch size is $T_0=1/\epsilon^{3/8}$ and $m=1/\epsilon^{1.625}$, then the total computational complexity is $O(1/\epsilon^{3.625})$.
\end{thm}
\begin{proof}	
	Define $\gamma_t=\min\left(\frac{\beta^a}{\|\z_t\|^{1/2}}, \frac{\beta^a}{\epsilon_0}\right)$ with $\epsilon_0=2\beta^a$. Then we know that $\eta_t=\alpha\gamma_t$ and $\gamma_t\leq \frac{1}{2}$. Note that $\alpha\leq\frac{1}{L}$, so we have $\eta_t\leq \frac{1}{2L}$.
	By the $L$-smoothness of $F$, we have
	\begin{equation}
	\label{fastrate1:eq1}
	\begin{aligned}
	&\quad F(\w_{t+1})\leq F(\w_t)+\nabla^\top F(\w_t)(\w_{t+1}-\w_t)+\frac{L}{2}\left\|\w_{t+1}-\w_t\right\|^2\\
	&\leq F(\w_t)-\eta_t\nabla^\top F(\w_t)\z_t+\left(\frac{\eta_t^2L}{2}+\frac{\gamma_t}{2L}\right)\|\z_t\|^2-\frac{1}{2L}\gamma_t\|\z_t\|^2\\
	&= F(\w_t)-\eta_t\nabla^\top F(\w_t)\left(\z_t-\nabla F(\w_t)+\nabla F(\w_t)\right)+\left(\frac{\eta_t^2L}{2}+\frac{\gamma_t}{2L}\right)\|\z_t\|^2-\frac{1}{2L}\gamma_t\|\z_t\|^2\\
	&\stackrel{(a)}{\leq} F(\w_t)-\eta_t\nabla^\top F(\w_t)\left(\z_t-\nabla F(\w_t)+\nabla F(\w_t)\right)+\left(\eta_t^2L+\frac{\gamma_t}{L}\right)\left(\|\z_t-\nabla F(\w_t)\|^2+\left\|\nabla F(\w_t)\right\|^2\right)-\frac{1}{2L}\gamma_t\|\z_t\|^2\\
	&\stackrel{(b)}{\leq} F(\w_t)-\frac{\eta_t}{2}\left\|\nabla F(\w_t)\right\|^2 +\frac{\eta_t}{2}\|\z_t-\nabla F(\w_t)\|^2+\left(\eta_t^2L+\frac{\gamma_t}{L}\right)\left(\|\z_t-\nabla F(\w_t)\|^2+\left\|\nabla F(\w_t)\right\|^2\right)-\frac{1}{2L}\gamma_t\|\z_t\|^2\\
	&=F(\w_t)-\left(\frac{\eta_t}{2}-\eta_t^2 L-\frac{\gamma_t}{L}\right)\left\|\nabla F(\w_t)\right\|^2 +\left(\eta_t^2L+\frac{\gamma_t}{L}+\frac{\eta_t}{2}\right)\|\z_t-\nabla F(\w_t)\|^2-\frac{1}{2L}\gamma_t\|\z_t\|^2\\
	&\stackrel{(c)}{\leq} F(\w_t)-\frac{1}{2L}\gamma_t\|\z_t\|^2+\frac{1}{L}\left\|\z_t-\nabla F(\w_t)\right\|^2,
	\end{aligned}	
	\end{equation}
	where (a) holds since $\|\z_t\|^2\leq 2\|\z_t-\nabla F(\w_t)\|^2+2\|\nabla F(\w_t)\|^2$, (b) holds since $-\nabla^\top F(\w_t)\z_t\leq \frac{1}{2}\left(\|\nabla F(\w_t)\|^2+\|\z_t-\nabla F(\w_t)\|^2\right)$, (c) holds due to $\frac{\eta_t}{2}-\eta_t^2 L-\frac{\gamma_t}{L}\geq 0$ (since $\eta_t\leq\frac{1}{2L}$, we have $\frac{\eta_t}{2}-\eta_t^2 L\geq \frac{1}{2L}$ and note that $\frac{\gamma_t}{L}\leq \frac{1}{2L}$). 
	
	By the definition of $\gamma_t$, we have
	\begin{equation}
	\label{fastrate1:eq2}
	\begin{aligned}
	\gamma_t\|\z_t\|^2&\geq \beta^{2a}\|\z_t\|\min\left(\frac{\|\z_t\|^{1/2}}{\beta^a}, \frac{\|\z_t\|}{\beta^a\epsilon_0}\right)=\beta^{2a}\|\z_t\|\min\left(\frac{\|\z_t\|^{1/2}}{\beta^a}, \frac{\|\z_t\|}{2\beta^{2a}}\right)\\
	&\stackrel{(a)}{\geq} \beta^{2a}\|\z_t\|\left(\frac{\|\z_t\|^{1/2}}{\beta^a}-\frac{1}{2}\right)=\beta^a\|\z_t\|^{3/2}-\frac{\beta^{2a}\|\z_t\|}{2},
	\end{aligned}	
	\end{equation} 
	where (a) holds since $x\geq x-\frac{1}{2}$, $\frac{x^2}{2}\geq x-\frac{1}{2}$ hold for any $x$ and let $x=\frac{\|\z_t\|^{1/2}}{\beta^a}$.
	
	Combining~(\ref{fastrate1:eq1}) and~(\ref{fastrate1:eq2}), we have
	\begin{equation*}
	\begin{aligned}
	\beta^a\|\z_t\|^{3/2}&\leq \gamma_t\|\z_t\|^2+\frac{\beta^{2a}\|\z_t\|}{2}
	\leq 2L\left(F(\w_t)-F(\w_{t+1})\right)+\frac{\beta^{2a}\|\z_t\|}{2}+2\|\z_t-\nabla F(\w_t)\|^2\\
	&=2L\left(F(\w_t)-F(\w_{t+1})\right)+\beta^{a}\|\z_t\|^{3/2}\cdot\frac{\beta^a}{2\|\z_t\|^{1/2}}+2\|\z_t-\nabla F(\w_t)\|^2.
	\end{aligned}
	\end{equation*}
	If $\frac{\beta^{a}}{2\|\z_t\|^{1/2}}\leq \frac{1}{2}$, we have $	\beta^a\|\z_t\|^{4/3}\leq 4L\left(F(\w_t)-F(\w_{t+1})\right)+4\|\z_t-\nabla F(\w_t)\|^2.$
	If $\frac{\beta^{a}}{2\|\z_t\|^{1/2}}> \frac{1}{2}$, then $\beta^a>\|\z_t\|^{1/2}$, and hence we have $\beta^a\|\z_t\|^{3/2}\leq \beta^{4a}$. As a result, we have
	\begin{equation}
	\label{fastrate1:eq3}
	\begin{aligned}
	\beta^a\|\z_t\|^{3/2}\leq  4L\left(F(\w_t)-F(\w_{t+1})\right)+4\|\z_t-\nabla F(\w_t)\|^2+\beta^{4a}.
	\end{aligned}
	\end{equation}
	Taking summation on both sides of~(\ref{fastrate1:eq3}) over $t=1,\ldots,T$ yields
	\begin{equation}
	\label{fastrate1:eq4}
	\sum_{t=1}^{T}\|\z_t\|^{3/2}\leq  4L\sum_{t=1}^{T}\frac{F(\w_t)-F(\w_{t+1})}{\beta^a}+\sum_{t=1}^{T}\frac{4}{\beta^a}\|\z_t-\nabla F(\w_t)\|^2+\beta^{3a}T.
	\end{equation}
	Define $\Delta_t=\z_t-\nabla F(\w_t)$, then we have
	\begin{equation}
	\label{fastrate1:eq5}
	\|\nabla F(\w_t)\|^{3/2}\leq 2\|\z_t\|^{3/2}+2\|\Delta_t\|^{3/2}.
	\end{equation}
	Hence, we have
	\begin{equation}
	\label{fastrate1:eq6}
	\begin{aligned}
	&\quad\sum_{t=1}^{T}\|\z_t\|^{3/2}+\|\nabla F(\w_t)\|^{3/2} \\
	&\stackrel{(a)}{\leq }2\sum_{t=1}^{T}\|\Delta_t\|^{3/2}+12L\sum_{t=1}^{T}\frac{F(\w_t)-F(\w_{t+1})}{\beta^a}+\sum_{t=1}^{T}\frac{12}{\beta^a}\|\z_t-\nabla F(\w_t)\|^2+3\beta^{3a}T\\
	&\stackrel{(b)}{\leq} \sum_{t=1}^{T}\frac{3}{2}\left(\frac{\|\Delta_t\|^2}{\beta^a}+\frac{\beta^{3a}}{3}\right)+12L\sum_{t=1}^{T}\frac{F(\w_t)-F(\w_{t+1})}{\beta^a}+\sum_{t=1}^{T}\frac{12}{\beta^a}\|\z_t-\nabla F(\w_t)\|^2+3\beta^{3a}T\\
	&\leq 12L\sum_{t=1}^{T}\frac{F(\w_t)-F(\w_{t+1})}{\beta^a}+\sum_{t=1}^{T}\frac{14}{\beta^a}\|\z_t-\nabla F(\w_t)\|^2+4\beta^{3a}T,
	\end{aligned}
	\end{equation}
	where (a) holds due to~(\ref{fastrate1:eq4}) and~(\ref{fastrate1:eq5}), (b) holds because $\min_{x>0}\frac{c^2}{x}+\frac{x^3}{3}=\frac{4c^{3/2}}{3}$.
	
	By Lemma~\ref{lemma:1}, we know that
	\begin{equation*}
	\begin{aligned}
	\E\left[\delta_{t+1}^2\right] &\leq \left(1- \frac{\beta}{2}\right)\E\left[\delta_t^2\right] +2\beta^2 \sigma^2 + \E\left[\frac{CL_H^2\eta_t^4\|\z_t\|^4}{\beta^3}\right]\\
	&\stackrel{(a)}{\leq}\left(1- \frac{\beta}{2}\right)\E\left[\delta_t^2\right] +2\beta^2 \sigma^2 + \E\left[\frac{CL_H^2\alpha^4\beta^{4a}\|\z_t\|^4}{\max(\|\z_t\|^{2},\epsilon_0^4)\beta^3}\right]\\
	&\leq \left(1- \frac{\beta}{2}\right)\E\left[\delta_t^2\right] +2\beta^2 \sigma^2 + \E\left[CL_H^2\alpha^4\beta^{4a-3}\|\z_t\|^{2}\right].
	\end{aligned}
	\end{equation*}
	Note that $CL_H^2\alpha^4\leq \frac{1}{14G^{1/2}}$, we have
	\begin{equation}
	\label{fastrate1:eq7}
	\frac{\beta}{2}\E\left[\delta_t^2\right]\leq \E\left[\delta_t^2-\delta_{t+1}^2\right] +2\beta^2 \sigma^2 + \E\left[\frac{\beta^{4a-3}\|\z_t\|^{2}}{14G^{1/2}}\right].
	\end{equation}
	Taking summation on both sides of~(\ref{fastrate1:eq7}) over $t=1,\ldots,T$, we have
	\begin{equation}
	\label{fastrate1:eq8}
	\begin{aligned}
	\sum_{t=1}^{T}\E\left[\delta_t^2\right]&\leq \frac{\E\left[\delta_1^2\right]}{\beta}+2\beta\sigma^2T+\sum_{t=1}^{T}\E\left[\frac{\beta^{4a-4}\|\z_t\|^{2}}{14}\right] \\
	&=\frac{\E\left[\delta_1^2\right]}{\beta}+2\beta\sigma^2T+\sum_{t=1}^{T}\E\left[\frac{\beta^{a}\|\z_t\|^{2}}{14G^{1/2}}\right]\\
	&\leq \frac{\E\left[\delta_1^2\right]}{\beta}+2\beta\sigma^2T+\sum_{t=1}^{T}\E\left[\frac{\beta^{a}\|\z_t\|^{3/2}}{14}\right],
	\end{aligned}
	\end{equation}
	where the equality holds since $a=4/3$ and last inequality holds since $\|\z_t\|\leq G$.
	
	Taking expectation on both sides of~(\ref{fastrate1:eq6}) and combining~(\ref{fastrate1:eq8}), we have
	\begin{equation*}
	\begin{aligned}
	\sum_{t=1}^{T}&\E\left[\|\z_t\|^{3/2}+\|\nabla F(\w_t)\|^{3/2}\right] \\
	&\leq \frac{12L(F(\w_1)-F_*)}{\beta^a}+\frac{14\E\left[\delta_1^2\right]}{\beta^{1+a}}+\frac{28\beta\sigma^2 T}{\beta^a}+\sum_{t=1}^{T}\E\left[\|\z_t\|^{3/2}\right]+4\beta^{3a}T.
	\end{aligned}
	\end{equation*}
	As a result, we have
	\begin{equation*}
	\frac{1}{T}\sum_{t=1}^{T}\E\left[\|\nabla F(\w_t)\|^{3/2}\right]\leq \frac{12L(F(\w_1)-F_*)}{\beta^a T}+\frac{14\E\left[\delta_1^2\right]}{\beta^{1+a}T}+\frac{28\beta\sigma^2 }{\beta^a}+4\beta^{3a}. \qedhere
	\end{equation*}
\end{proof}

\section{Related Work}
\label{sec:relatedwork}
\paragraph{Adaptive Gradient Methods} Adaptive gradient methods were first proposed in the framework of online convex optimization~\citep{duchi2011adaptive,mcmahan2010adaptive}, which dynamically incorporate knowledge of the geometry of the data to perform more informative gradient-based learning. This type of algorithm was proved to have fast convergence if stochastic gradients are sparse~\citep{duchi2011adaptive}. Based on this idea, several other adaptive algorithms were proposed to train deep neural networks, including Adam~\citep{kingma2014adam}, Amsgrad~\citep{reddi2019convergence}, RMSprop~\citep{tieleman2012lecture}. There are many work trying to analyze variants of adaptive gradient methods in both convex and nonconvex case~\citep{chen2018closing,chen2018universal,chen2018convergence,luo2019adaptive,chen2018closing,chen2018convergence,ward2019adagrad,li2019convergence,chen2018universal}. 
Notably, all of these works are able to establish faster convergence rate than SGD, based on the assumption that stochastic gradients are sparse. However, this assumption may not hold in deep learning. In contrast, our algorithm can have faster convergence than SGD even if stochastic gradients are not sparse, since our algorithm's new data-dependent adaptive complexity does not rely on the sparsity of stochastic gradients.

\paragraph{Variance Reduction Methods} Variance reduction is a technique to achieve fast rates for finite sum and stochastic optimization problems. It was first proposed for finite-sum convex optimization~\citep{johnson2013accelerating} and then it was extended in finite-sum nonconvex~\citep{allen2016variance,reddi2016stochastic,zhou2018stochastic} and stochastic nonconvex~\citep{lei2017non,fang2018spider,wang2019spiderboost,pham2020proxsarah,cutkosky2019momentum} optimization. To prove faster convergence rate than SGD, all these works make the assumption that the objective function is an average of individual functions and each one of them is smooth. In contrast, our analysis does not require such an assumption and to achieve a faster-than-SGD rate.
\paragraph{Other Related Work} ~\citet{arjevani2019lower} show that SGD is optimal for stochastic nonconvex smooth optimization, if one does not assume that every component function is smooth. There are recent work trying to establish faster rate than SGD, when the Hessian of the objective function is Lipschitz~\citep{fang2019sharp,cutkosky2020momentum}. There are several empirical papers, including LARS~\citep{you2017scaling} and LAMB~\citep{you2019large}), which utilize both moving average and normalization for training of deep neural networks with large-batch sizes. \cite{zhang2020complexity} consider an algorithm for finding stationary point for nonconvex nonsmooth problems.~\cite{levy2017online} considers convex optimization setting and design algorithms which adapts to the smoothness parameter.~\cite{liu2019variance} introduced Rectified Adam to alleviate large variance at the early stage.
However, none of them establish data-dependent adaptive complexity as in our paper.

\end{document}